\documentclass[12pt]{article}
\usepackage{csquotes}
\usepackage{natbib}
\usepackage[utf8]{inputenc} 
\usepackage{url}            
\usepackage{booktabs}       
\usepackage{amsfonts}       
\usepackage{nicefrac}       
\usepackage{microtype}      
\usepackage{algorithm}
\usepackage{grffile}
\usepackage{algpseudocode}
\usepackage{amsmath,amsfonts,amsthm} 
\usepackage{amssymb}
\usepackage{graphicx,psfrag,epsf}
\usepackage{listings}
\usepackage{enumerate}
\theoremstyle{empty}
\usepackage{color}
\newtheorem{thm}{Theorem}
\newtheorem{lemma}{Lemma}
\usepackage{lscape}
\usepackage{bbm}
\usepackage{rotating}
\usepackage{lipsum,multicol}
\usepackage{comment}

\usepackage{xr-hyper}
\usepackage[colorlinks = true,
           linkcolor = blue,
           urlcolor  = magenta,
           citecolor = blue,
           anchorcolor = blue]{hyperref}   
\def\Bern{\textsf{Bern}}

\def\DE{\textsf{DE}}

\newcommand{\vertiii}[1]{{\left\vert\kern-0.25ex\left\vert\kern-0.25ex\left\vert #1 
    \right\vert\kern-0.25ex\right\vert\kern-0.25ex\right\vert}}

\newcommand{\blind}{1}

\begin{document}

\def\spacingset#1{\renewcommand{\baselinestretch}%
{#1}\small\normalsize} \spacingset{1}

\spacingset{1.48}  
\if1\blind 
{
  \title{\bf  Bayesian Regularization for Graphical Models with Unequal Shrinkage}
  \author{Lingrui Gan, 
    Naveen N. Narisetty, 
Feng Liang\\
    Department of Statistics, University of Illinois at Urbana-Champaign\\
    \date{}
}
  \maketitle
} \fi

\if0\blind
{
  \bigskip
  \bigskip
  \bigskip
  \begin{center}
    {\LARGE\bf Title}
\end{center}
  \medskip
} \fi

\bigskip

\begin{abstract}
We consider a Bayesian framework for estimating a high-dimensional sparse precision matrix, in which  adaptive shrinkage and sparsity are induced by a mixture of Laplace priors. Besides discussing our formulation from the Bayesian standpoint, we investigate the MAP (maximum  a posteriori) estimator from a penalized likelihood perspective that gives rise to a new non-convex penalty approximating the $\ell_0$ penalty. Optimal error rates for estimation consistency in terms of  
various matrix norms along with selection consistency for sparse structure recovery are shown for the unique MAP estimator under mild conditions. For fast and efficient computation, an EM algorithm is proposed to compute the MAP  estimator of the precision matrix and (approximate) posterior probabilities on the edges of the underlying sparse structure. Through extensive simulation studies and a real application to a call center data, we have demonstrated the fine performance of our method compared with existing alternatives.
\end{abstract}

\noindent%
{\it Keywords:}  {\small precision matrix estimation, sparse Gaussian graphical model, spike-and-slab priors, Bayesian regularization}   \vfill


\section{Introduction}

Covariance matrix and precision matrix (inverse of the covariance matrix) are among the most fundamental quantities in Statistics as they describe the dependence between different variables (components) of a multivariate observation. Not surprisingly, they play pivotal roles in many statistical problems including graphical models, classification, clustering, and regression, which are used extensively in many application areas including biological, engineering, and finance. Take the Gaussian graphical model (GGM) as an example. The precision matrix provides great insight into the conditional dependence structure in a graph, since the conditional independence of $i$-th and $j$-th variables of an undirected Gaussian Markov random field is equivalent to the $(i,j)$-th entry of the precision matrix being zero, {see a recent review by \citet{pourahmadi2013high}}. Such results have helped researchers to identify complex network structures in applications such as high-throughput biological data, {for example, in \citet{wille2004sparse}. }

Estimating the precision matrix, especially under the high dimensional setting where the variable dimension $p$ can possibly be larger than the sample size $n$, is a particularly challenging problem. Given the current prevalence of high dimensional data and the wide utility of precision matrix, this problem has received significant attention in  recent literature. When the sample covariance matrix is positive definite, its inverse is a natural estimator for the precision matrix. However, the inverse of sample covariance matrix as an estimator is demonstrated to have poor performance  in numerous studies \citep{johnstone2001distribution,paul2007asymptotics,pourahmadi2013high}. 
Moreover, when $p >n$, the precision matrix estimation problem is ill-posed without further restricting assumptions. {One of the most commonly used assumptions to remedy this issue is to assume that the precision matrix is sparse, i.e., a large majority of its entries are zero \citep{dempster1972covariance}, which turns out to be quite useful in practice in the aforementioned GGM owing to its interpretability. Another possibility is to assume a sparse structure on the covariance matrix through, for example, a sparse factor model \citep{carvalho2008high, fan2008high, fan2011high, buhlmann2011statistics, pourahmadi2013high, rovckova2016fast}, to obtain a sparse covariance matrix estimator, and invert it to estimate the precision matrix. However, the precision matrix estimator obtained from this strategy is not guaranteed to be sparse, which is important for interpretability in our context. }%

Regularization provides a general framework for dealing with high dimensional problems. There are two major approaches that utilize regularization to estimate the precision matrix and its sparse structure. 

The first one is \emph{regression based approach} where a sparse regression model is estimated separately for each column to identify and estimate the nonzero elements of that column in the precision matrix $\Theta$ \citep{meinshausen2006high,peng2009partial,Zhou09,khare2015convex}. This approach focuses more on the sparse selection of the entries, and the estimated precision matrix is generally not positive definite.

The other is \emph{likelihood based approach} which aims to optimize the negative log-likelihood function \eqref{eq:loglik} together with an element-wise penalty term on $\Theta$ \citep{yuan2007model,banerjee2008model,friedman2008sparse,fan2009network}.  Among these methods, Graphical Lasso  (GLasso) \citep{friedman2008sparse} is the most commonly used owing to its scalability. GLasso estimator for the precision matrix is also not guaranteed to be positive definite. \cite{mazumder2012graphical} proposed algorithms that modify GLasso and ensure positive definiteness of the estimated precision matrix. Apart from these two general approaches, regularization can be applied with other forms of loss functions, an example of which is the CLIME estimator proposed by \cite{cai2011constrained}. 


{Theoretical properties of the likelihood based methods for Gaussian graphical models have been studied in the literature.} In \citet{rothman2008sparse}, \citet{lam2009sparsistency} and \citet{loh2015regularized}, estimation error rates in Frobenius norm have been established for likelihood based estimators with Lasso and SCAD penalties. For GLasso, stronger results in entrywise maximum norm are obtained by \cite{ravikumar2011high} under a restrictive assumption on $\Theta$, called the irrepresentable assumption, { when the multivariate distribution of the observations has an exponential tail (such as sub-Gaussian distributions). A slower rate is shown when the distribution has a polynomial tail (such as $t$-distributions with sufficiently large degrees of freedom).} Similar results on estimation error rate in maximum norm are shown by \cite{loh2014support} for non-convex penalized estimators under sub-Gaussian distributions but their results require beta-min conditions. \cite{cai2011constrained} provide such results for CLIME estimator both under exponential and polynomial tails with the assumption that all the absolute column sums of $\Theta$ are bounded. 
 
The precision matrix estimation problem is less studied under the Bayesian framework possibly due to the high computational cost associated with MCMC when $p$ is large. \cite{marlin2009sparse} proposed a Bayesian model and a variational Bayes algorithm for GGMs with a block structure. \cite{wang2012bayesian} proposed a Bayesian version of GLasso and the associated posterior computation algorithms. \cite{carvalho2009objective}, \cite{dobra2011bayesian}, \cite{Wangli2012} and \cite{Mohammadi2015} used G-Wishart priors and proposed stochastic search methods for the computation. \cite{banerjee2015bayesian} studied a Bayesian approach with mixture prior distributions that have a point-mass and a Laplace distribution. They provide posterior consistency results and a computational approach using Laplace approximation. With the exception of \cite{banerjee2015bayesian}, theoretical properties of Bayesian methods for sparse precision matrix estimation have not been studied. The results of \cite{banerjee2015bayesian} are on estimation error rate in Frobenius norm similar to those of \cite{rothman2008sparse}, but assume the underlying distribution to be Gaussian. 

In this paper, we propose a new Bayesian approach for estimation and structure recovery for GGMs. Specifically, to achieve adaptive shrinkage, we model the off-diagonal elements of $\Theta$ using a continuous spike-and-slab prior with a mixture of two Laplace distributions, {which is known as the spike-and-slab Lasso prior in \citet{rockova2015bayesian}, \citet{rovckova2016fast} and \citet{rovckova2016spike}}. Continuous spike-and-slab priors are commonly used for high dimensional regression  \citep{George93, Ishwaran05,Narisetty14} and a Gibbs sampling algorithm is often used for posterior computation. However, such a Gibbs sampler for our problem has an extremely high computational burden and instead {we propose a novel EM algorithm for computation, which is motivated by the EM algorithm  for linear regression from \citet{rovckova2014emvs} and the one  for factor models from \cite{rovckova2016fast}.} Our novel computational and theoretical contributions in the paper are summarized as follows: 
\begin{itemize}
\item We propose a new approach for precision matrix estimation, named BAGUS, short for 
``\textbf{BA}yesian regularization for 
\textbf{G}raphical models with \textbf{U}nequal \textbf{S}hrinkage." The adaptive (unequal) shrinkage is due to the non-convex penalization by our Bayesian formulation. 

\item Although the Gaussian likelihood is used in our Bayesian formulation, our theoretical results hold beyond GGMs. We have shown that our procedure enjoys the optimal estimation error rate of $O_p\left(\sqrt{\frac{\log p}{n}} \right)$ in the entrywise maximum norm and selection consistency under both exponential and polynomial tail distributions with very mild conditions. Our theoretical result is stronger than the best existing result by \cite{cai2011constrained}, as we assume boundedness of $\Theta$ in operator norm which is weaker than the assumption of bounded absolute column sum of $\Theta$.

\item We propose a fast EM algorithm which produces a maximum a posteriori (MAP) estimate of the precision matrix and (approximate) posterior probabilities on all edges that can be used to learn the graph structure. The EM algorithm has computational complexity comparable to the state-of-the-art GLasso algorithm \citep{mazumder2012graphical}. 

\item Our algorithm is guaranteed to produce a symmetric and positive definite estimator unlike many existing estimators including CLIME. 
\end{itemize}

The remaining part of the paper is organized as follows. In Section 2, we present our model and prior set-up in the Bayesian framework along with a discussion on its penalized likelihood perspective. In Section 3, we provide our theoretical consistency results followed by the details of the EM algorithm in Section 4. Section 5 presents numerical results in extensive simulation studies and a real application for predicting telephone center call arrivals. Proofs, technical details, and  R code used for empirical results can be found in Online Supplementary Material. 

\subsection*{Notation} 

For a $p \times q$ matrix $A=[a_{ij}]$, we denote its Frobenius norm  by $\|A\|_{F}=\sqrt{\sum_{(i,j)}a_{ij}^2}$,  the entrywise $\ell_{\infty}$ norm (i.e., maximum norm)  $ \| A \|_{\infty}={\max_{(i,j)}|a_{ij}|}$, and  its spectral norm by $ \| A \|_2=\sup \{ \| A\mathbf{x}\|: \mathbf{x} \in \mathbb{R}^q, \|\mathbf{x} \|\le 1\}$ where $\| \mathbf{x}\|$ denotes the $l_2$ norm of vector $\mathbf{x}$.  
 For a
$p \times p$ square matrix $A$, let $A^{-}$ denote the off-diagonal elements of $A$, $A^{+}$  the diagonal elements of $A$, and $\lambda_{\min}(A)$ and $\lambda_{\max}(A)$ the smallest and the largest eigenvalues, respectively. For a square symmetric matrix $A$, {its spectral norm is equal to its maximum eigenvalue, that is,  $\| A \|_2 = \lambda_{\max}(A)$, and}
its maximum absolute column sum (i.e., the $\ell_{1}/\ell_{1}$ operator norm) is the same as its maximum absolute row sum (i.e., the $\ell_{\infty}/\ell_{\infty}$ operator norm), denoted by $\vertiii{A}_{\infty}={\max_{1\le j\le q}\sum_{i=1}^{p}|a_{ij}|}$. 

Let $\Theta^0=[\theta^{0}_{ij}]$ and $\Sigma^0=[\sigma^{0}_{ij}]$ denote the true precision matrix and covariance matrix, respectively.  Let ${S^0}=\{(i,j):\theta_{ij}^0\ne0 \}$ denote the index set of all nonzero entries in $\Theta^0$ and ${{S^0}^{c}}$ is its complement.
Define ${\theta^0_{\max}}=\max_{ij}|\theta^0_{ij}|$ and $M_{\Sigma^0}=\vertiii{\Sigma^0}_{\infty}$. Define $\Gamma=\Theta^{-1}\otimes\Theta^{-1}$ as the Hessian matrix of  $g:=-\log\det(\Theta)$. {$\Gamma_{(j,k),(l,m)}$ corresponds to the second partial derivative $\frac{\partial^2 g}{\partial \theta_{jk}\partial \theta_{lm}}$, and for any two subsets $T_1$ and $T_2$ of $\{(i,j):1\le i,j\le p\}$, we use $\Gamma_{T_1T_2}$ to denote the matrix with rows and columns of $\Gamma$ indexed by $T_1$ and $T_2$ respectively.}
We further denote $M_{\Gamma^0}=\vertiii{{\Gamma^{0}}^{-1}_{{S^0S^0}}}_{\infty}=\vertiii{({\Theta^0}\otimes{\Theta^0})_{{S^0S^0}}}_{\infty}$. Define the column sparsity $d=\underset{{i=1,2,...,p}}{\max} card\{j: \theta_{ij}^{0}\ne 0\}$ and the off-diagonal sparsity $s = card({S^0}) - p$, where $card$ denotes the cardinality of the set in its argument.

\section{Bayesian Regularization for Graphical Models} 

Our data consist of a random sample of $n$ observations $Y_1, \dots, Y_n$ which are assumed to be $iid$ $p$-variate random vectors following a multivariate distribution with mean zero and precision matrix $\Theta$. In short, we use the following notation:
$$Y_1, \dots, Y_n \overset{iid}{\sim} N(0,\Theta^{-1}).$$ Our primary goal is to estimate $\Theta$ and identify the sparse structure in the elements of $\Theta$. For the Bayesian framework, we work with the Gaussian $\log$-likelihood given by
\begin{equation} \label{eq:loglik}
\ell(\Theta) = \log f(Y_1, \dots, Y_n | \Theta)=\frac{n}{2}\Big(\log \det (\Theta) - \text{tr} (S\Theta)\Big)
\end{equation}
where $S=[s_{ij}] = \frac{1}{n}\sum Y_i Y_i^t$ denotes the sample covariance matrix of the data. We note that in spite of working with the Gaussian likelihood, we allow the observations to have non-Gaussian distributions including those with polynomial tails.

\subsection{Bayesian Formulation} 
Next we describe our prior specification on the following two groups of parameters: the diagonal entries $\{\theta_{ii} \}$ and the off diagonal entries, where the latter is reduced to the upper triangular entries $\{\theta_{ij}: i<j \}$  due to symmetry. 

On the upper triangular entries $\theta_{ij}$ ($i < j$), we place the following spike-and-slab prior, {known as the spike-and-slab Lasso prior developed in a series of work by \citet{rockova2015bayesian}, \citet{rovckova2016fast} and \citet{rovckova2016spike}}:
\begin{equation} \label{def:SS:prior}
\pi(\theta_{ij}) = \frac{\eta}{2v_1} \exp \Big \{  -\frac{|\theta_{ij}|}{v_1} \Big \}   + \frac{1-\eta}{2v_0}  \exp \Big \{  -\frac{|\theta_{ij}|}{v_0} \Big \} ,
\end{equation}
which is a mixture of two Laplace distributions of different scales $v_0$ and $v_1$ with $v_1 > v_0 > 0.$ The mixture distribution (\ref{def:SS:prior}), represents our prior on $\theta_{ij}$ which could take values of relatively large magnitude modeled by the Laplace distribution with scale parameter $v_1$ (i.e., the \enquote{slab} component), or which could take values of very small magnitude modeled by the Laplace distribution with scale parameter $v_0$ (i.e., the \enquote{spike} component). 
In the traditional spike-and-slab prior, the \enquote{spike} component is set to be a point mass at zero, which corresponds to our setting with $v_0=0.$ Here we use a continuous version of the spike-and-slab prior, in which $v_0$ is set be nonzero but relatively small compared with $v_1$. Continuous spike-and-slab priors with normal components were proposed by \cite{George93} in the linear regression context and their high dimensional shrinkage properties were studied by \cite{Ishwaran05} and \cite{Narisetty14}. {\citet{rockova2015bayesian} and \citet{rovckova2016spike} considered the spike-and-slab Lasso prior given by \eqref{def:SS:prior} for linear regression and 
studied the adaptive shrinkage property of such priors as well as various asymptotic properties concerning the posterior mode.} An advantage of continuous spike-and-slab priors is that the continuous prior distributions on $\theta_{ij}$ allow the use of efficient algorithms that do not require switching the active dimension of the parameter. 

For the diagonal entries $\theta_{ii}$ of  the precision matrix, a weakly informative Exponential prior is specified since $\theta_{ii}$  do not need to be shrunk to zero: $$ \pi(\theta_{ii}) = \tau\exp(-\tau\theta_{ii})   \mathbbm{1}(\theta_{ii} > 0).$$ 

Although $\Theta$ can be fully parameterized by these two groups of parameters, they are not independent as the determinant of $\Theta$ needs to be positive. Therefore, the support for the joint prior distribution on elements of $\Theta$ is restricted such that $\Theta$ is positive definite, i.e., $\Theta\succ 0.$ In addition, we constrain the spectral norm of $\Theta$ to be upper bounded: $\| \Theta\|_2\le B$. Such a constraint is not very restrictive since it often appears in the assumptions for theoretical studies of precision matrix estimation anyway: a large spectral norm of $\Theta$ implies high correlation among variables, a setup in which most methods fail. An important consequence of this constraint will be discussed in Section \ref{sec:convexity}.

So our prior distribution on $\Theta$ is given by
\begin{equation} \label{eq:overall:prior}
\pi(\Theta)=  \prod_{i<j} \pi(\theta_{ij})  \prod_i \pi(\theta_{ii})  \mathbbm{1}(\Theta\succ 0)\mathbbm{1}( \| \Theta \|_2\le B).
\end{equation}

\subsection{The Penalized Likelihood Perspective}\label{sec:penalize}
 
If estimation of $\Theta$ is of main interest, then a natural choice is the MAP estimator $\tilde{\Theta}$ that maximizes the posterior distribution $\pi(\Theta | Y_1, \cdots, Y_n)$. This is equivalent to minimizing the following objective function under the constraint $\|\Theta\|_2\le B$ and $\Theta\succ 0$:
\begin{eqnarray}
L(\Theta) &=& -\log\pi(\Theta | Y_1, \cdots, Y_n) \nonumber \\
&= &  -\ell(\Theta) - \sum_{i<j}\log \pi(\theta_{ij} | \eta) - \sum_{i}\log\pi(\theta_{ii}|\tau)  + \text{Const.}\nonumber \\
&= &  \frac{n}{2} \Big (\text{tr}(S\Theta) -\log \det(\Theta) \Big ) + \sum_{i<j} \text{pen}_{SS}(\theta_{ij}) + \sum_{i} \text{pen}_1(\theta_{ii}) + \text{Const.} \label{eq:Loss}
\end{eqnarray}
where
\begin{equation}
\text{pen}_{SS}(\theta) = -\log \left [ \Big(\frac{\eta}{2v_1}\Big) e^{-\frac{|\theta|}{v_1}} +\Big(\frac{1-\eta}{2v_0} \Big) e^{-\frac{|\theta|}{v_0}}\Big) \right ] \label{eq:pen_SS}
\end{equation}
and $\text{pen}_1(\theta)= \tau|\theta|. $

If viewed from the penalized likelihood perspective, the objective function $L(\Theta)$ employs two penalty functions, induced by our Bayesian formulation. The penalty function on the diagonal entries, $\text{pen}_1(\theta)$, is the same as the Lasso penalty. The hyperparameter $\tau$ is suggested to be small, so the Lasso penalty mainly shrinks the estimates of $\theta_{ii}$ instead of truncating them to be zero.  

More importantly, the penalty function on the off-diagonal entries, $\text{pen}_{SS}(\theta)$, coming from the spike-and-slab prior has an interesting shrinkage property. To highlight the difference between this penalty and the Lasso penalty, we plotted them in Figure \ref{fig:pencom}. We also compare our spike-and-slab penalty with the spike-and-slab penalty that arises by using a mixture of two normal distributions \citep{george1997approaches} instead of Laplace distributions: 
\[ \text{pen}_{NSS}(\theta) = -\log \left [ \Big(\frac{\eta}{\sqrt{2 \pi v_1}}\Big) e^{-\frac{\theta^2}{2 v_1}} +\Big(\frac{1-\eta}{\sqrt{2 \pi v_0}} \Big) e^{-\frac{\theta^2}{2 v_0}}\Big) \right ], 
\]
where ``$NSS$" in the subscript stands for normal spike-and-slab prior. In Figure \ref{fig:pencom}, we set $v_0=0.1$ and $v_1=10$ for both $\text{pen}_{SS}(\theta)$ and $\text{pen}_{NSS}(\theta)$. Also, we subtract their values at $0$ so the corresponding penalty at $\theta=0$ is zero. We can see that the penalty function we use, $\text{pen}_{SS}(\theta)$, provides the best continuous approximation of the $L_0$ penalty among the three.

To gain more insight about the penalty functions, we plot the derivatives/subgradient of the spike-and-slab penalty $\text{pen}_{SS}(\theta)$ in Figure \ref{fig:devcom}. 
A simple calculation reveals that
\begin{equation} \label{eq:bias}
\frac{\partial}{\partial |\theta|} \text{pen}_{SS}(\theta) =\frac{1}{v_1} \frac{\frac{\eta}{2v_1}e^{-\frac{|\theta|}{v_1}}}{\pi(\theta)}  +\frac{1}{v_0} \frac{\frac{1-\eta}{2v_1}e^{-\frac{|\theta|}{v_0}}}{\pi(\theta)} = \frac{w (\theta)}{v_1} + \frac{1-w (\theta)}{v_0},
\end{equation}
which is a weighted average of $1/v_1$ and $1/v_0$ with the weight $w (\theta)$ being the conditional probability of $\theta$ belonging to the ``slab'' component \citep{rovckova2016spike}. 
Recall that the derivative of a penalty function should ideally have its maximum at zero and then decay gradually to 0 (asymptotically), because a non-decreasing derivative with respect to $|\theta|$ leads to a bias and affects the performance in finite sample settings \citep{fan2001variable,loh2014support}. 
This is the case with  $\text{pen}_{SS}(\theta)$: 
As $|\theta|$ becomes larger, the mixing weight $w$ gets larger, which leads to a smooth transition from a large penalty $1/v_0$ produced from the ``spike'' component, to a smaller penalty $1/v_1$  from the ``slab'' component. From Figure \ref{fig:devcom}, we can see that $\text{pen}_{NSS}(\theta)$ does not have this desired property, and neither does the Lasso penalty.


 
\begin{figure}[!htbp]
\begin{center}
\includegraphics[width=1\linewidth,height=0.27\linewidth]{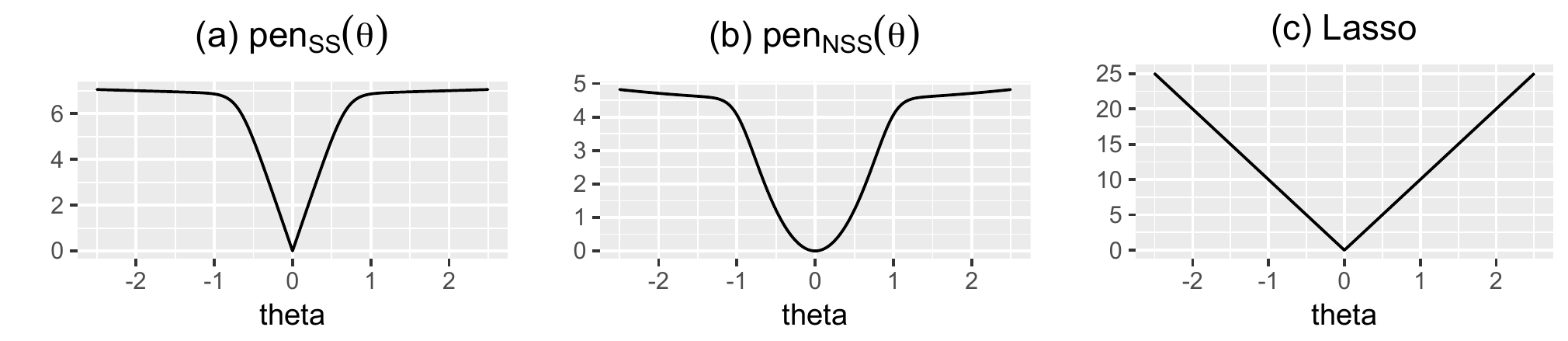}
\end{center}
\caption{Plot of different penalty functions. (a):  penalty induced from the spike-and-slab prior with a mixture of Laplace distributions; (b):  penalty induced from the spike-and-slab  prior with a mixture of normal distributions; (c):  Lasso penalty. }\label{fig:pencom}
\end{figure}

\begin{figure}[!htbp]
\begin{center}
\includegraphics[width=1\linewidth]{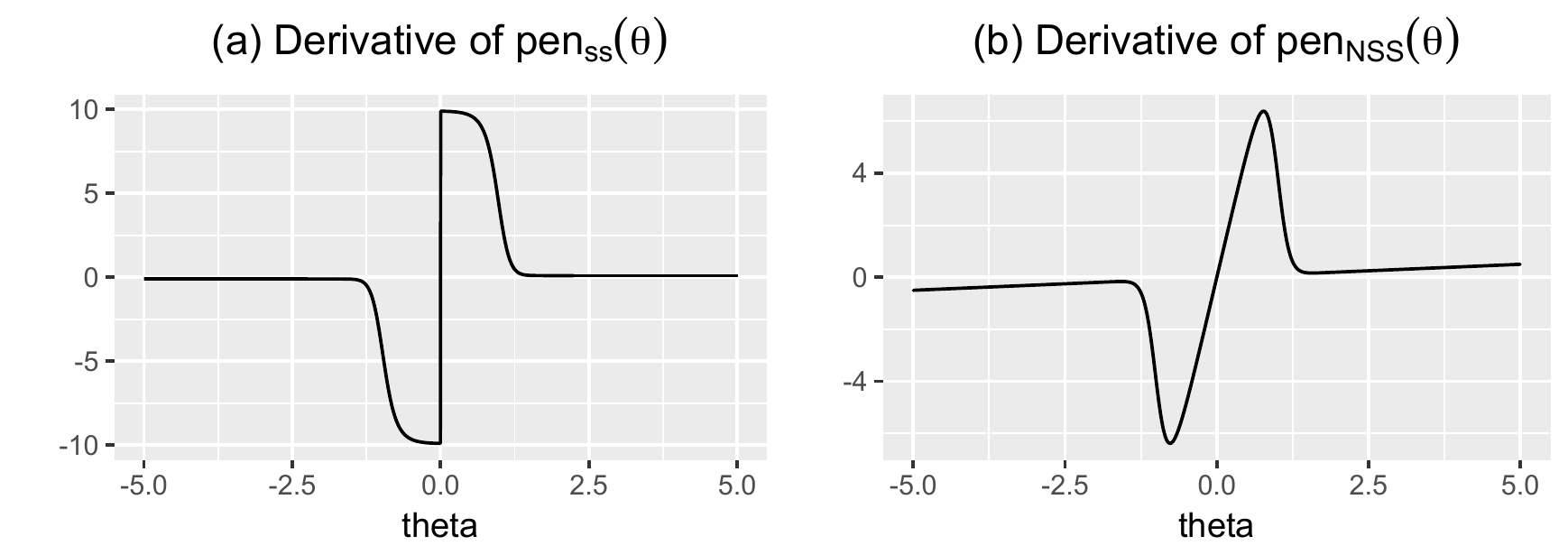}
\end{center}
\caption{Plot of the derivative/subgradient of the penalty functions
}
\label{fig:devcom}
\end{figure}

\subsection{Posterior Maximization and Local Convexity}\label{sec:convexity}
The non-convexity of our spike-and-slab penalty $\text{pen}_{SS}(\theta)$ leads to desired shrinkage and selection behavior, but it could bring additional computation challenges as the posterior objective function $L(\Theta)$ is no longer convex and may have multiple local optima. However, this is not a problem in our case with the upper bound on the spectral norm of $\Theta$ (\ref{eq:overall:prior}). More specifically, the following theorem ensures that the optimization of $L(\Theta)$ with the spectral norm constraint is a convex optimization problem, that is, locally within the spectral norm ball, we are dealing with convex optimization resulting in a unique MAP estimate. This result is motivated by Lemma 6 from \cite{loh2014support}.

\begin{thm} \label{lemma:convex}
If $B < (2nv_0)^{\frac{1}{2}}$, then $\min_{\Theta\succ 0,\|\Theta\|_2\le B}L(\Theta)$
is a strictly convex problem.
\end{thm}

\begin{proof}
Decompose $L(\Theta)$ as the sum of the following two terms: $-\ell(\Theta)-\frac{1}{8v_0}\|\Theta\|_F^2$ and $ \sum_{i<j} \text{pen}_{SS}(\theta_{ij}) + \sum_{i} \text{pen}_1(\theta_{ii}) +\frac{1}{8v_0}\|\Theta\|_F^2$. We prove the theorem by checking that the second-order subgradient of each term in the decomposition of $L(\Theta)$ is positive which would imply that both the terms are strictly convex.

The second-order subgradient of the first term is given by $ -\nabla^2\ell(\Theta) - \frac{1}{4v_0}$, where $-\nabla^2\ell(\Theta)=\frac{n}{2}(\Theta\otimes\Theta)^{-1}$. The smallest eigenvalue of $-\nabla^2\ell(\Theta) $ can be bounded as:
$$\lambda_{\min} \left(-\nabla^2\ell(\Theta) \right) - =\frac{n}{2}\lambda_{\max}^{-1}(\Theta\otimes\Theta)  =\frac{n}{2}\lambda_{\max}^{-2}(\Theta) > \frac{1}{4v_0},$$
{ where the last inequality is because $\|\Theta\|_2\le B\le(2nv_0)^{\frac{1}{2}}$ implies that $\lambda^2_{\max}(\Theta)\le 2nv_0$ and leads to $\frac{n}{2}\lambda_{\max}^{-2}(\Theta)\ge\frac{1}{4v_0}.$ Therefore, $-\nabla^2\ell(\Theta) - \frac{1}{4v_0}$ is strictly convex.}

We now consider the second-order subgradient of $\text{pen}_{SS}(\theta_{ij})$:
\begin{equation*}
|\text{pen}_{SS}{''}(\theta_{ij})|=\frac{(\frac{1}{v_0}-\frac{1}{v_1})\frac{\eta v_0}{(1-\eta)v_1}e^{\frac{\theta_{ij}}{v_0}-\frac{\theta_{ij}}{v_1}}}{(\frac{\eta v_0}{(1-\eta)v_1}e^{\frac{\theta_{ij}}{v_0}-\frac{\theta_{ij}}{v_1}}+1)^2} \leq \frac{1}{4} \left(\frac{1}{v_0}-\frac{1}{v_1} \right) <\frac{1}{4v_0},
\end{equation*}
where the first inequality is because for any $x$, $\frac{|x|}{(1+|x|)^2} \leq \frac{1}{4}$. This implies that the second term in the decomposition of $L(\Theta)$ is also strictly convex and the theorem is proved.
\end{proof}

\subsection{Uncovering the Sparse Structure} \label{sec:sparse_structure}

In many applications, identifying the zero entries in $\Theta$ (referred to as structure estimation or graph selection) is also of major interest along with the estimation of $\Theta$. Inference on the latent sparse structure of $\Theta$ or equivalently the sparse structure of a graph can be directly induced from our spike-and-slab prior. We can re-express the spike-and-slab prior (\ref{def:SS:prior}) as the following two-level hierarchical prior:
\begin{equation} 
 \left \{ \begin{array}{lcl} 
\theta_{ij} \mid r_{ij}=0 & \sim &  \DE(0,v_0) \\
 \theta_{ij} \mid r_{ij}=1 & \sim &  \DE(0,v_1) 
\end{array} \right.  \label{eq:spike:slab} 
\end{equation}
where  $r_{ij}$ follows
\begin{equation}
\label{prior:r_ij}
r_{ij} \mid \eta \ \sim \Bern(\eta).
\end{equation}
Here $\DE(0, v)$ denotes the double exponential (Laplace)  distribution with scale $v$ and and $\Bern(\eta)$ denotes the Bernoulli distribution with probability $\eta$ .

We can view the binary variable $r_{ij}$ as the indicator for the sparsity pattern: $r_{ij}
=1$ implies $\theta_{ij}$ being the \enquote{signal} (i.e., from the slab component), and $r_{ij}=0$ implies $\theta_{ij}$ being the \enquote{noise} (i.e., from the spike component). {In the fully Bayesian approach, the posterior inclusion probability for an edge connecting $i$ and $j$ is given by }
\[ \mathbb{P}(r_{ij}=1 | Y_1, \dots, Y_n ) = \int \mathbb{P}(r_{ij}=1 | \theta_{ij} ) \pi(\theta_{ij} | Y_1, \dots, Y_n) d \theta_{ij}, 
\] 
{
which is the integrated probability of $\theta_{ij}$ being from the slab component (corresponding to $\gamma_{ij}=1$) with respect to the posterior distribution of $\theta_{ij}$. In our analysis, we approximate this probability by using the MAP estimator $\tilde{\Theta}$ as follows:}
\begin{equation} \label{eq:p_ij}
p_{ij} =  \mathbb{P}(r_{ij}=1 | \tilde{\theta}_{ij} ) = \frac{ \left ( \frac{\eta}{2v_1} \right) e^{-\frac{|\tilde{\theta}_{ij}|}{v_1}}}{\left (\frac{\eta}{2v_1}\right) e^{-\frac{|\tilde{\theta}_{ij}|}{v_1}} +\left (\frac{1-\eta}{2v_0} \right) e^{-\frac{|\tilde{\theta}_{ij}|}{v_0}} }. 
 \end{equation}
We can then threshold $p_{ij}$ to identify the edges: if $p_{ij}$ is greater than a pre-specified threshold such as 0.5, then the $(i,j)$ pair is identified as an edge. 
 
Denote $\mathbb{P}(r_{ij}=1 | \tilde{\theta}_{ij} = 0)$ by $p^{\star}(0)$. The quantity $\frac{1}{p^{\star}(0)}-1 = v_1 (1-\eta)/(v_0 \eta)$ represents the interplay of all the parameters $(v_0, v_1, \eta)$ and it plays an important role both in our asymptotic analysis for precision matrix estimation that will be presented in the next section, and also in the analysis of \citet{rovckova2016spike} and \citet{rockova2015bayesian} for high-dimensional linear regression.

\section{Theoretical Results}
 
Let $\tilde{\Theta}$ denote the MAP estimator, the unique minimizer of the loss function (\ref{eq:Loss}). In this Section, we provide theoretical results on the estimation accuracy of $\tilde{\Theta}$. We also show that the structure selected based on thresholding the posterior probabilities $p_{ij}$ matches the true sparse structure with probability going to one. 

\subsection{Conditions}
 
\subsubsection{Tail Conditions on the Distribution of $Y$}

In our analysis, we do not restrict to the situation where the true distribution of $Y$ is Gaussian. Instead, we provide analysis for two cases according to the tail conditions on the true distribution of a $p$-variate random vector $Y = (Y^{(1)},Y^{(2)},...,Y^{(p)})$. 

\begin{enumerate}
\item[(C1)] Exponential tail condition: Suppose that there exists some $0<\eta_1<1/4$ such that $\frac{\log p}{n}<\eta_1$ and 
\begin{equation}
Ee^{t{Y^{(j)}}^2}\le K \text{ for all }|t|\le\eta_1, \text{ for all } j=1, \dots, p
\end{equation}
where $K$ is a bounded constant.
\item[(C2)]Polynomial tail condition: Suppose that for some $\gamma$, $c_1>0$, $p\le c_1n^\gamma$, and for some $\delta_0>0$,
\begin{equation}
E|Y^{(j)}|^{4\gamma+4+\delta_0}\le K, \text{ for all } j=1, \dots, p.
\end{equation}
\end{enumerate}
Note that when $Y$ follows a Gaussian or a sub-Gaussian distribution, condition (C1) is satisfied. { When $p=n$, condition (C2) is satisfied for $t$-distributions with degrees of freedom greater than $8$. When $p=n^2$, condition (C2) is satisfied  for $t$-distributions with degrees of freedom greater than $12$.} The same tail conditions are also considered by \cite{cai2011constrained} and \cite{ravikumar2011high}.

\subsubsection{Conditions on $\Theta^0$}

We make the following assumption on the true precision matrix $\Theta^0$ for studying estimation accuracy.  
\begin{enumerate}
 	\item[(A1)] $\lambda_{\max}(\Theta^0)\le 1/k_1<\infty$ or equivalently $0<k_1\le\lambda_{\min}(\Sigma^0)$,
	{ where $k_1$ is some constant greater than $0$.}
\end{enumerate}
{Note that because the largest eigenvalue of $\Theta^0$ is bounded, all the elements of $\Theta^0$ are bounded, and cannot grow with $p$ and $n$.}

In addition, we make the minimum signal assumption below for studying sparse structure recovery. 
\begin{enumerate}
	 \item[(A2)] The minimal \enquote{signal} entry satisfies $\label{cond:C}
	{\underset{(i,j)\in {S^0}}{\min}{|\theta^0_{ij}|}}\geq K_0 {\sqrt{\frac{\log p}{n}}}$, 
	where $K_0>0$ is a sufficiently large constant not depending on $n$.
\end{enumerate}

Similar and in some cases stronger assumptions are imposed in other theoretical analysis of precision matrix estimation and sparse structure recovery \citep{rothman2008sparse, lam2009sparsistency, ravikumar2011high, cai2011constrained, loh2014support}. For a comparison of various theoretical results, see the discussion in Section \ref{sec:compare}. 



\subsection{Theoretical Results}
The following theorem gives estimation accuracy under  the entrywise $\ell_{\infty}$ norm. In particular, the following theorem implies that with an appropriate choice of $(v_0, v_1, \eta, \tau)$ and $B$, we could achieve the $O_p\left(\sqrt{\frac{\log p}{n}}\right)$ error rate for distributions with an exponential or a polynomial tail.
\begin{thm}(Estimation accuracy in entrywise $\ell_{\infty}$ norm)\label{Thm:estimate}\\
Assume condition (A1) holds. For any pre-defined constants $C_3>0$, $\tau_0>0$, define $C_1=\eta_1^{-1}(2+\tau_0+\eta_1^{-1}K^2)$ when the exponential tail condition (C1) holds, and  $C_1=\sqrt{({\theta^0_{\max}}+1)(4+\tau_0)}$ when  the polynomial tail condition (C2) holds. Assume that\\
i) the prior hyper-parameters $v_0,v_1,\eta,$ and $\tau$ satisfy
\begin{equation}
\begin{cases}
\label{eq:hyper}
\frac{1}{nv_1}={C_3}\sqrt{\frac{\log p}{n}}(1-\varepsilon_1), ~~ \frac{1}{nv_0}>C_4\sqrt{\frac{\log p}{n}} \\
\frac{v_1^2(1-\eta)}{v_0^2\eta}\le p^\varepsilon, ~\text{ and } ~ \tau\le C_3\frac{n}{2}\sqrt{\frac{\log p}{n}}
\end{cases}
\end{equation}
for some constants $ \varepsilon_1>0$, $C_4 > C_3$ and some sufficiently small $\varepsilon$, \\
ii) the spectral norm $B$ satisfies $\frac{1}{k_1}+2d(C_1+C_3)M_{\Gamma^0}\sqrt{\frac{\log p}{n}}<B<(2nv_0)^{\frac{1}{2}}$, and\\
iii) the sample size $n$ satisfies $\sqrt{n}\ge M\sqrt{\log p}$,\\
~~~where $M=\max\Big\{2d(C_1+C_3)M_{\Gamma^0}{\max\Big({3M_{\Sigma^0}},{3M_{\Gamma^0}{M_{\Sigma^0}}^3},\frac{2}{k_1^2}\Big)},\frac{2C_3\varepsilon_1}{k_1^2}\Big\}$.\\
Then,  the MAP estimator $\tilde{\Theta}$ satisfies
\begin{equation} \label{eq:elementwise:l_infinity}
\|\tilde{\Theta}-\Theta^0\|_{\infty}\le2(C_1+C_3)M_{\Gamma^0}\sqrt{\frac{\log p}{n}}.
\end{equation}
with probability greater than $1-\delta_1$, where $\delta_1=2p^{-\tau_0}$ when condition (C1) holds,  and $\delta_1=O(n^{-\delta_0/8}+p^{-\tau_0/2})$ when condition (C2) holds. 
\end{thm}

Theorem \ref{Thm:estimate} shows that the estimation error of our MAP estimator $\tilde{\Theta}$ can be controlled through an interplay between the parameters $(v_0,v_1,\eta, \tau)$ and $B$. To help readers understand this result, we provide an explanation of the required conditions.

In our proof, the term $\frac{1}{n} \text{pen}{'}_{SS}(\theta)$, which decreases from $1/(n v_0)$ to $1/(n v_1)$ when $|\theta|$ increases from zero to infinity, serves as an adaptive thresholding value. The conditions in \eqref{eq:hyper} ensure the following properties of this adaptive thresholding rule: 1) to eliminate noise, $1/(nv_0)$ is set to be bigger than $\sqrt{(\log p)/n}$, the typical noise level in high-dimensional analysis; 2) to reduce bias due to thresholding, $1/(nv_1)$ is set to be of a constant order of $\sqrt{(\log p)/n}$, or much smaller  by varying $\varepsilon_1$; 3) the thresholding level should be close to $1/(n v_1)$ when $\theta$ is of a certain order bigger than the noise level $\sqrt{(\log p)/n}$, which is ensured by the upper bound on $\frac{v_1^2(1-\eta)}{v_0^2 \eta}$.

The upper bound on $B$ in condition $ii)$ is to ensure that our objective function $L(\Theta)$ is strictly convex. However, $B$ cannot be too small, otherwise, even if $L(\Theta)$ is convex, the constrained local mode cannot achieve the desired estimation accuracy $\|\tilde{\Theta}-\Theta^0\|_{\infty}=O_p \Big (\sqrt{\log p/n} \Big )$.

When $M_{\Gamma^0}, M_{\Sigma^0}$ remain constant as a function of $(n,p,d)$, Theorem 2 guarantees that with proper tuning, an estimation error bound of $O(\sqrt{\log p/n})$ in $\ell_\infty$ norm can be achieved for the MAP estimator $\tilde{\Theta}$ with high probability. Similar results can be found in \citet{ravikumar2011high} and \citet{loh2014support} when $M_{\Gamma^{0}}, M_{\Sigma^{0}}$ are constants. { If $M_{\Gamma^0}, M_{\Sigma^0}$ are of the order $O(p)$, then we require the sample size $n$ to grow faster than the order $O(p)$.}


Theorem \ref{Thm:estimate} follows from  a more general result stated as Theorem \ref{thm:proof} in Appendix  A from the Online Supporting Material. The specific definition for $C_4$ and the one for $\varepsilon$ are also provided in Theorem \ref{thm:proof} in the Online Supporting Material.

We now present the following result on estimation accuracy of $\tilde{\Theta}$ in terms of Frobenius norm, spectral norm and $\ell_{\infty}/\ell_{\infty}$ operator norm. This result is based on Theorem \ref{Thm:estimate} and Lemma \ref{lemma:5} from Appendix A. 

\begin{thm}(Estimation accuracy in other norms)\label{thm:corollary}\\
Under the same  conditions of Theorem \ref{Thm:estimate}, \\
(i) if the exponential tail condition (C1) holds, then
\begin{equation}
\begin{split}
& \| \tilde{\Theta}-\Theta^0 \|_{F}<2\Big(\eta_1^{-1}(2+\tau_0+\eta_1^{-1}K^2)+C_3\Big)M_{\Gamma^0}\sqrt{\frac{(p+s)\log p}{n}},\\
&\vertiii{\tilde{\Theta}-\Theta^0}_{\infty}, \| \tilde{\Theta}-\Theta^0\|_{2}<2\Big(\eta_1^{-1}(2+\tau_0+\eta_1^{-1}K^2)+C_3\Big)M_{\Gamma^0}\min\{d,\sqrt{p+s}\}\sqrt{\frac{\log p}{n}}, 
\end{split}
\end{equation}
with probability greater than $1-2p^{-\tau_0}$; \\
(ii) if the polynomial tail condition (C2) holds, then
\begin{equation} \label{eq:cond:elementwies:l_infinity}
\begin{split}
&\| \tilde{\Theta}-\Theta^0 \|_{F}<2(\sqrt{({\theta^0_{\max}}+1)(4+\tau_0)}+C_3)M_{\Gamma^0}\sqrt{\frac{(p+s)\log p}{n}},\\
&\vertiii{\tilde{\Theta}-\Theta^0}_{\infty},\|\tilde{\Theta}-\Theta^0\|_{2}<2(\sqrt{({\theta^0_{\max}}+1)(4+\tau_0)}+C_3)M_{\Gamma^0}\min\{d,\sqrt{p+s}\}\sqrt{\frac{\log p}{n}},
\end{split}
\end{equation}
with probability greater than $1-O(n^{-\delta_0/8}+p^{-\tau_0/2}).$
\end{thm}

Next, we discuss selection consistency for the sparse structure before providing a comparison of our results with the existing results in Section \ref{sec:compare}. 

As discussed in Section \ref{sec:sparse_structure}, we propose to estimate ${S^0}$, the set of nonzero elements of $\Theta$, by thresholding the inclusions probability $p_{ij}$ that is defined at (\ref{eq:p_ij}). The following theorem shows that {$\hat{S^0}=\{(i,j): p_{ij}\ge T\}$}, the set of edges with posterior probability greater than $T$, is a consistent estimator of ${S^0}$  for any $0<T<1$. 

\begin{thm}(Selection consistency)
Assume the same conditions in Theorem \ref{Thm:estimate} and condition (A2) with the following restriction:
\begin{equation}
\label{eq:gap}
\epsilon_0< \frac{1}{\log p}\log \left(\frac{v_1(1-\eta)}{v_0\eta} \right)<(C_4-C_3)\big ( K_0-2(C_1+C_3)M_{\Gamma^0} \big )
\end{equation}
for some arbitrary small constant $\epsilon_0>0$. Then, for any $T$ such that $0<T<1$, we have
\begin{equation*}
\mathbb{P} \Big ( \hat{{S^0}}={S^0} \Big ) \rightarrow 1.
\end{equation*}\label{Thm:select}
\vspace{-2ex}
\end{thm}

A proof of Theorem \ref{Thm:select} is provided in Appendix B. 

{In our model, sparsity is induced by an interplay between the parameters $v_0,v_1$ and $\eta$ through $\log\left(v_1\left(1-\eta\right)/(v_0\eta)\right)$. When $\log\left(v_1\left(1-\eta\right)/(v_0\eta)\right)$ falls in the gap mentioned in Equation \eqref{eq:gap}, the selection consistency can be achieved.} 

\subsection{Comparison with Existing Results}\label{sec:compare}

We compare our results with those of GLasso \citep{ravikumar2011high}, CLIME \citep{cai2011constrained} and the non-convex regularization based method by \cite{loh2014support}. 

 In \cite{ravikumar2011high}, the irrepresentable condition, $\vertiii{\Gamma_{{{S^0}^c}{S^0}}\Gamma_{{S^0}{S^0}}^{-1}}_{\infty} \le 1-\alpha$, is needed to establish the rate of convergence in entrywise $\ell_{\infty}$ norm. Such an assumption is quite  restrictive, and is not needed for our results. In addition, under the polynomial tail condition, the rate of convergence established in \cite{ravikumar2011high} is $O_p \left (\sqrt{\frac{p^c}{n}}\right)$, slower than  our rate $O_p \left (\sqrt{\frac{\log p}{n}} \right)$. 

The theoretical results for CLIME  \citep{cai2011constrained} are similar to ours in terms of estimation accuracy. However, the main difference is the assumption on $\Theta^0$. We assume boundedness of the largest eigenvalue of $\Theta^0$, which is strictly weaker than the boundedness of $\vertiii{\Theta^0}_\infty$ (the $\ell_{\infty}/\ell_{\infty}$ operator norm), the assumption imposed for CLIME. The weakness of our assumption follows from H\"older's inequality. To illustrate the strict difference between these assumptions, we consider the following precision matrix as an example: 
\begin{equation} \label{eq:star:structure}
\theta^0_{ii} = 1,   \forall i; \quad  \theta^0_{1,i}= \theta_{i,1} = \frac{1}{\sqrt{p}}, \text{ if } i \ne 1;  \quad \theta_{ij}^0=0 \text{ if } i \ne j \text{ and } i \ne 1.  
\end{equation}
The precision matrix above has the so-called star structure, which is frequently observed in networks with a hub. In Figure \ref{star}, we plot the maximum eigenvalue and the maximum of the absolute row sum of this matrix with varying dimension $p$. We can see that it is easy to satisfy the upper bound on maximum eigenvalue, but not the upper bound on the $\ell_{\infty}/\ell_{\infty}$ operator norm, since the latter is diverging with $p$. 
\begin{figure}[!htbp]
\begin{center}
\includegraphics[width=0.8\linewidth]{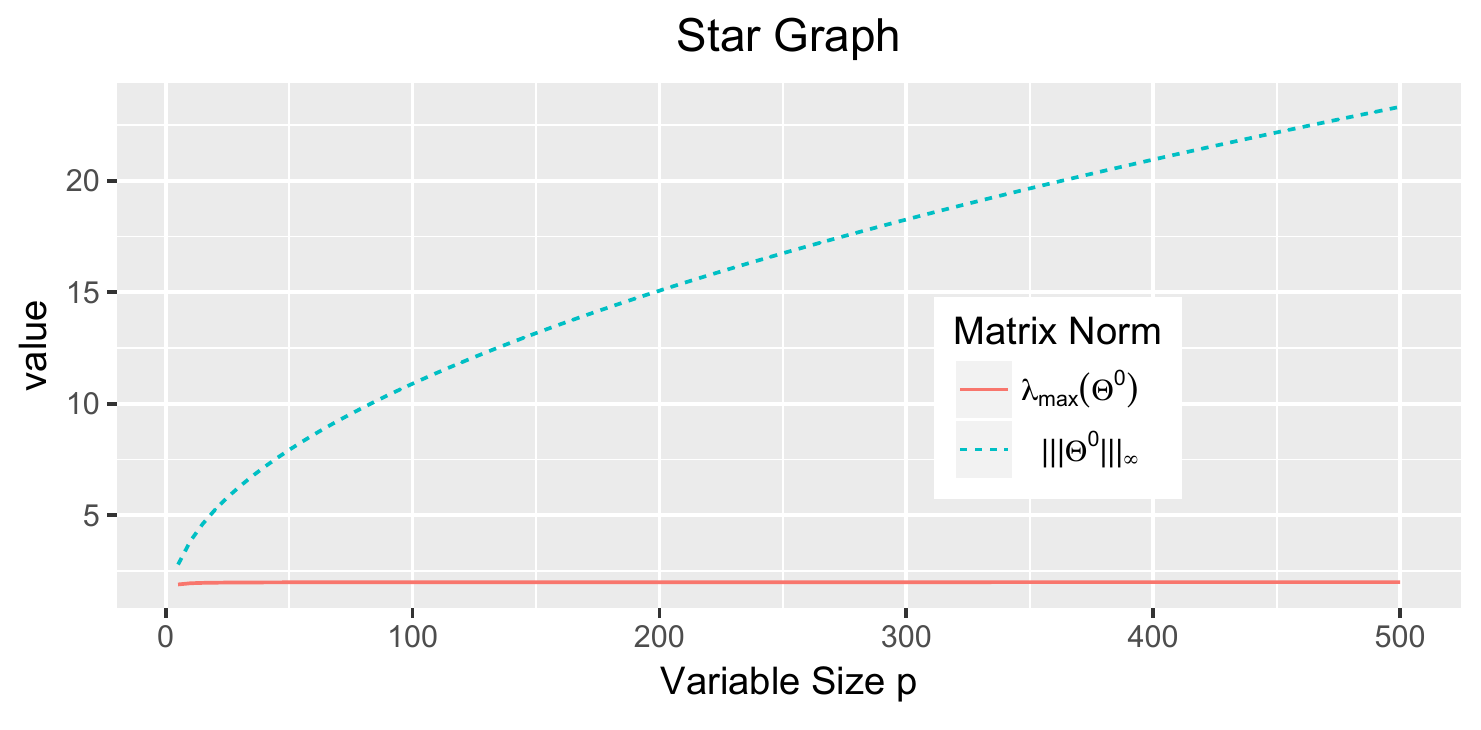}
\end{center}
\caption{Plots of the maximum eigenvalue (solid line) and the $\ell_{\infty}/\ell_{\infty}$ operator norm (dashed line) for precision matrices with the star structure (\ref{eq:star:structure}). {Our model assumption corresponds to an upper bound on the solid line, while the one for CLIME corresponds to an upper bound on the dashed line.}}\label{star}
\end{figure}

The major difference between our results and those from \cite{loh2014support} is also in the weakness of the assumptions. The beta-min condition (minimal signal strength) is needed for the rate of estimation accuracy established in \cite{loh2014support}, while we do not require this assumption for estimation consistency. In addition, their results are only available for sub-Gaussian distributions, while we consider a much broader class of distributions, i.e., distributions with exponential or polynomial tails.

\section{Computation with EM Algorithm}

We now describe how to compute the MAP estimate $\tilde{\Theta}.$ Directly optimizing the negative log of the posterior distribution (\ref{eq:Loss}) is not easy. One numerical complication comes from the penalty term (\ref{eq:pen_SS}): it has a summation inside the logarithm due to the mixture prior distribution on $\theta_{ij}.$ The expectation-maximization (EM) algorithm is a popular tool in handling such a complication. 

Recall the two-level hierarchical representation of the prior on $\theta_{ij}$ introduced  in Section \ref{sec:sparse_structure}. Define $R$ as the $p \times p$ matrix with binary entries $r_{ij}$. Then the full posterior distribution $\pi(\Theta,R | Y_1, \cdots, Y_n)$ is proportional to
\begin{equation} \label{eq:joint}
f(Y_1, \dots, Y_n |\Theta) \cdot \Big [ \prod_{i<j}\pi(\theta_{ij}|r_{ij})\pi(r_{ij}|\eta) \Big ] \cdot \Big [ \prod_{i}\pi(\theta_{ii}|\tau)\Big]\mathbbm{1}(\Theta\succ 0)\mathbbm{1}(\|\Theta\|_2\le B).
\end{equation}
We treat $R$ as latent and derive an EM algorithm to obtain the MAP estimate of $\Theta$ from the M-step and the posterior distribution of $R$ from the E-step upon convergence. The E-step of our algorithm is inspired by the {EM algorithm for linear regression from \citet{rovckova2014emvs} and the one for factor models from \cite{rovckova2016fast}}, and the M-step of our algorithm is inspired by the optimization procedure used by GLasso \citep{banerjee2008model,  friedman2008sparse,mazumder2012graphical}. 


\subsection{The E-step}

At the E-step, we first compute the distribution of $R$ given  the parameter value from the previous iteration $\Theta^{(t)}$. Note that the binary indicator $r_{ij}$ does not appear in the likelihood function, and only appears in \eqref{eq:spike:slab} and \eqref{prior:r_ij} in the prior specification. It is easy to show that $r_{ij} \mid \Theta^{(t)}, Y_1, \dots, Y_n$ follows $\Bern(p_{ij})$ with
\begin{equation} \label{P}
 \log \frac{p_{ij}}{1 - p_{ij}} = \log\frac{v_0}{v_1}+\log\frac{\eta}{1 - \eta} -\frac{|\theta_{ij}^{(t)}|}{v_1}+\frac{|\theta_{ij}^{(t)}|}{v_0}. 
\end{equation}

Next we evaluate the expectation of $\log \pi(\Theta,R | Y_1, \cdots, Y_n)$ with respect to $\pi(R | \Theta^{(t)}, Y_1, \dots, Y_n)$, which gives rise to the so-called $Q$ function: 
 \begin{equation}
  \begin{aligned}
\label{Q}
  Q(\Theta|\Theta^{(t)}) 
  &=  \Big\{\frac{n}{2} \log \det(\Theta)-\frac{n}{2} \text{tr}(S\Theta) + \sum_{i}(\log\tau-\tau\theta_{ii}) \\
&+\sum_{i<j}p_{ij} \Big [ -\log(2v_1)-\frac{|\theta_{ij}|}{v_1} +\log\eta \Big ] \\
&+\sum_{i<j}(1-p_{ij})\Big [ -\log(2v_0)-\frac{|\theta_{ij}|}{v_0} +\log(1-\eta) \Big ] \Big\}\mathbbm{1}(\Theta\succ 0)\mathbbm{1}(\|\Theta\|_2\le B).
\end{aligned}
\end{equation}

\subsection{The M-step}

At the M-step of the $(t+1)$th iteration, we sequentially update $\Theta$ in a column by column fashion to maximize $Q(\Theta|\Theta^{(t)}).$ 
Without loss of generality,  we describe the updating rule for the last column of $\Theta$ while fixing the others.

For convenience, partition the covariance matrix $W$  and the precision matrix $\Theta$  as follows:
\[
W= \begin{bmatrix}
    W_{11}       & w_{12} \\
    w_{12}^T       &w _{22}\\
\end{bmatrix}
\quad \quad
\Theta= \begin{bmatrix}
    \Theta_{11}       & \theta_{12} \\
    \theta_{12} ^T       & \theta_{22}\\
\end{bmatrix}
\]
where $W_{11}$ is the $(p-1) \times (p-1)$ sub-matrix, $w_{12}$ is the $(p-1) \times 1$ vector at the last column of $W$ and $w_{22}$ is the diagonal entry at the bottom-right corner. 
The sample covariance matrix $S$, the binary indicator matrix $R=[r_{ij}]$, and the conditional probability matrix $P=[p_{ij}]$ where $p_{ij}$ is defined in (\ref{P}) are also partitioned similarly. We list the following equalities from $W \Theta = \mathbf{I}_p$ which will be used in our algorithm: 
\begin{equation} \label{eq:W:Theta:equalities}
 \begin{bmatrix}
    W_{11}       & w_{12} \\
   \cdot       & w_{22}\\
\end{bmatrix}
=
 \begin{bmatrix}
    \Theta_{11}^{-1}+\frac{\Theta_{11}^{-1}\theta_{12}\theta_{12}^T\Theta_{11}^{-1}}{\theta_{22}-\theta_{12}^T\Theta_{11}^{-1}\theta_{12}}      & -\frac{\Theta_{11}^{-1}\theta_{12}}{\theta_{22}-\theta_{12}^T\Theta_{11}^{-1}\theta_{12}}   \\
 \cdot      &\frac{1}{\theta_{22}-\theta_{12}^T\Theta_{11}^{-1}\theta_{12}}
\end{bmatrix}.
\end{equation}

Given $\Theta_{11}$, to update the last column $(\theta_{12}, \theta_{22})$, we set the subgradient of $Q$ with respect to $(\theta_{12}, \theta_{22})$ to zero.  
First, take the subgradient of $Q$ with respect to $\theta_{22}$: 
\begin{equation}
\frac{\partial Q}{\partial \theta_{22}} =\frac{n}{2}\frac{1}{\theta_{22}-\theta_{12}^T\Theta_{11}^{-1}\theta_{12}}-\frac{n}{2} \left ( s_{22}+\tau \right ) =0.
  \label{eq:identity}
\end{equation}
Due to Equations \eqref{eq:W:Theta:equalities} and \eqref{eq:identity}, we have 
\begin{equation*}
w_{22} = \frac{1}{\theta_{22}-\theta_{12}^T\Theta_{11}^{-1}\theta_{12}} = s_{22}+\frac{2}{n}\tau,  
\end{equation*}
which leads to the following update for $\theta_{22}$:
\begin{equation}
\theta_{22} \leftarrow  \frac{1}{w_{22}}+\theta_{12}^T\Theta_{11}^{-1}\theta_{12}. \label{theta22}  
\end{equation}

Next take the subgradient of $Q$ with respect to $\theta_{12}$: \begin{equation}
\begin{split}
\frac{\partial Q}{\partial \theta_{12}}=& \frac{n}{2} \Big ( \frac{-2\Theta_{11}^{-1}\theta_{12}}{\theta_{22}-\theta_{12}^T\Theta_{11}^{-1}\theta_{12}}-2s_{12} \Big )-\Big(\frac{1}{v_1}p_{12}+\frac{1}{v_0}(1-p_{12})\Big)\odot \text{sign}(\theta_{12})\\
=&n(-\Theta_{11}^{-1}\theta_{12}w_{22}-s_{12})-\Big(\frac{1}{v_1}p_{12}+\frac{1}{v_0}(1-p_{12})\Big)\odot \text{sign}(\theta_{12})=0,
\label{der1}
\end{split}
\end{equation}
where $A \odot B$ denotes the element-wise multiplication of two matrices. Here the second line of  \eqref{der1} is due to the identities in (\ref{eq:W:Theta:equalities}). To update $\theta_{12}$, we then solve the following stationary equation for $\theta_{12}$ with coordinate descent, under the constraint $\|\Theta\|_2\le B$:
\begin{equation}
\begin{split}
&ns_{12}+nw_{22}\Theta_{11}^{-1}\theta_{12}+\Big(\frac{1}{v_1}P_{12}+ \frac{1}{v_0}(1-P_{12})\Big)\odot \text{sign}(\theta_{12})=0.
\end{split}
\label{eq:cor}
\end{equation}
The coordinate descent algorithm for updating $\theta_{12}$ is  summarized in Algorithm \ref{Algo2}.
Since only one column is changed, checking the bound $\|\Theta\|_2\le B$ is computationally feasible (see Appendix C in the Supplementary Material for more details). In practice, we could also proxy the constraint on $\|\Theta\|_2$ with a constraint on the largest absolute value of the elements in $\Theta$. In our empirical studies, this relaxation performs quite well. 

\begin{algorithm}[!htbp]
\begin{algorithmic}
\State \textbf{Initialize}  $\theta_{12}$ from the previous iteration as the starting point.
\Repeat

      \For{$j$ in $1:(p-1)$}  
 	\State Solve the following equation for ${\theta_{12}}_{j}$: 
	 $${ns_{12}}_j+nw_{22}{\Theta_{11}^{-1}}_{j,\setminus j}{\theta_{12}}_{\setminus j}+nw_{22}{\Theta_{11}^{-1}}_{j,j}{\theta_{12}}_j+\Big[\Big(\frac{1}{v_1}P_{12}+ \frac{1}{v_0}(1-P_{12})\Big)\odot \text{sign}(\theta_{12})\Big]_{j}=0.$$
	     \EndFor	     

\Until Converge or Max Iterations Reached.
\State If $\|\Theta\|_2> B:$ 
 \textbf{Return} $\theta_{12}$ from the previous iteration
 \State Else: \textbf{Return} $\theta_{12}$
  \end{algorithmic}
\caption{Coordinate Descent for $\theta_{12}$}
\label{Algo2}
\end{algorithm}

When updating $(\theta_{12}, \theta_{22})$, we need  $\Theta_{11}^{-1}$. Instead of directly computing the inverse of $\Theta_{11}$, we compute it from 
$$\Theta_{11}^{-1} =  W_{11}-w_{12}w_{21}/w_{22},$$
which is derived from  (\ref{eq:W:Theta:equalities}). 
After the update of $(\theta_{12}, \theta_{22})$ is completed, we ensure that $W \Theta = \mathbf{I}_p$ holds by updating $W_{11}$ and $w_{12}$ via identities from (\ref{eq:W:Theta:equalities}). Therefore, we always keep a copy of the most updated covariance matrix $W$ in our algorithm. { Note we don't update $w_{22}$ here, only because the relationship related to $w_{22}$ within $W \Theta = \mathbf{I}_p$ is already ensured. That is, if $w_{22}$ is updated using (\ref{eq:W:Theta:equalities}), it remains unchanged.}

\subsection{The Output}

The entire algorithm, BAGUS, is summarized and displayed as Algorithm \ref{Algo1}.  After convergence, we extract the following output from our algorithm: the $P$ matrix, the posterior probability on the sparse structure, from the E-step and the MAP estimator $\tilde{\Theta}$ from the M-step. 

\begin{algorithm}[!htbp]
\begin{algorithmic}
\State \textbf{Initialize} $W=\Theta$=$\mathbf{I}$
\Repeat

\State Update $P$ with each entry $p_{ij}$ updated as $
 \log \frac{p_{ij}}{1 - p_{ij}} \leftarrow\Big(\log\frac{v_0}{v_1}+\log\frac{\eta}{1 - \eta} -\frac{|\theta_{ij}^{(t)}|}{v_1}+\frac{|\theta_{ij}^{(t)}|}{v_0}\Big). $
      \For{$j$ in $1:p$}  
 	\State Move the $j$-th column and $j$-th row to the end (implicitly), namely $\Theta_{11}:=\Theta_{\setminus j \setminus j}$, $\theta_{12}:=\theta_{\setminus j  j}$, $\theta_{22}:=\theta_{jj}$
	\State Update $w_{22}$ using $w_{22} \leftarrow s_{22}+\frac{2}{n}\tau$
	\State Update $\theta_{12}$ by solving (\ref{eq:cor}) with Coordinate Descent for $\theta_{12}$.
         \State Update $\theta_{22}$  using $\theta_{22}\leftarrow  \frac{1}{w_{22}}+\theta_{12}^T\Theta_{11}^{-1}\theta_{12}.$
	\State Update {$W_{11}$, $w_{12}$} using (\ref{eq:W:Theta:equalities})
	     \EndFor	     

\Until Converge\\
 \textbf{Return} $\Theta$, $P$
  \end{algorithmic}
\caption{BAGUS}
\label{Algo1}
\end{algorithm}

To obtain an estimate of the sparse structure in $R$, we threshold the entries of $P$, namely: 
$$ \hat{r}_{ij} = 1, \text{ if } P_{ij} \ge 0.5; \quad \hat{r}_{ij} = 0, \text{ otherwise.}$$
As shown in Theorem \ref{Thm:select}, thresholding entries of $P$ with any number $T$ such that $0<T<1$ could recover the true sparse structure with probability converging to $1.$

For many existing algorithms, the positive definiteness of the estimate of $\Theta$ is not guaranteed.  For example,  GLasso  \citep{friedman2008sparse} can only ensure the positive definiteness of the estimate of the covariance matrix $W$, but not of the estimate of the precision matrix $\Theta$, as shown in  \cite{mazumder2012graphical}. The following theorem shows that MAP estimate $\tilde{\Theta}$ returned by our  algorithm  is ensured to be symmetric and positive definite.  
 \begin{thm}(Symmetry and positive definite) The  estimate of $\Theta$ returned by BAGUS is always symmetric, and it is also positive definite if the initial value $\Theta^{(0)}$ is positive definite. \label{thm:sym}
\end{thm}
 
A proof is given in the Supplementary Material.

\subsection{Remarks}

\begin{itemize}
\item[] \textbf{Computation Cost.}  In BAGUS, the computation cost is $O(p^2)$ for updating one column. There are $p$ columns in $\Theta$ to update, so the overall computational complexity of our algorithm is $O(p^3)$, which matches the computation cost for GLasso.

\item[] \textbf{Parameter Tuning.} BAGUS involves the following hyperparameters: $\eta$, $\tau$, $v_0$, and $v_1$. We always set $\eta = 0.5$ and $\tau = v_0$ so that there are only two parameters $v_0$ and $v_1$ to be tuned. { parameter tuning has an empirical Bayes flavor. In our simulations, we use the theoretical results to set the rough range of the hyper-parameters, and then use a BIC-like criterion to tune the hyper-parameters: } 
\begin{equation}
\begin{split}
&\text{BIC}=n\Big(\text{tr}(S\hat{\Theta})-\log \det(\hat{\Theta})\Big)+\log(n)\times\#\{(i,j):1\le i<j\le p, \hat{\theta}_{ij}\ne 0\}.
\end{split}
\end{equation}
The same BIC criterion is used by \cite{yuan2007model} while a similar BIC criterion with a regression based working likelihood is used by \cite{peng2009partial}.
\end{itemize}

\section{Empirical Results}\label{sec:emp}
In this section, we compare our method with the competitive alternatives in both simulated and real datasets and study the performance of our approach.

\subsection{Twelve Simulation Settings}
Following the simulation studies from related work \citep{yuan2007model,friedman2008sparse,peng2009partial,cai2011constrained}, we generate data $Y$  from a multivariate Gaussian distribution with mean $0$ and precision matrix $\Theta^0 = ( \theta_{ij}^0 )$. 

We consider four different models, i.e., four different forms of $\Theta^0$. The first three have been  considered in \cite{yuan2007model} and the fourth one is similar to the set-up in \cite{peng2009partial}. 

\begin{enumerate}
\item Model 1 (star model):  $\theta^0_{ii}$ = 1, $\theta^0_{1i}= \theta^0_{i1}$ = $\frac{1}{\sqrt{p}}$.
\item Model 2 ($AR(2)$ model):  $\theta^0_{ii}$ = 1, $\theta^0_{i,i-1}= \theta^0_{i-1,i}$ = 0.5 and $\theta^0_{i,i-2} = \theta^0_{i-2,i}$ = 0.25.
\item Model 3 (circle model):  $\theta^0_{ii}$ = 2, $\theta^0_{i,i-1}$ = $\theta^0_{i-1,i}$ = 1, and $\theta^0_{1p}=\theta^0_{p1}$=0.9.
\item Model 4 (random graph): The true precision matrix $\Theta^0$ is set as follows. 
\begin{enumerate}
\item Set $\theta^0_{ii}=1$.
\item Randomly select $1.5\times p$ of the off-diagonal entries $\theta^0_{ij}$ ($i \ne j$) and set their values to be uniform from $[0.4,1]\cup[-1,-0.4]$; {set the remaining off-diagonal entries to be zero.}
\item Calculate the sum of absolute values of the off-diagonal entries for each column, and then divide each off-diagonal entry by $1.1$ fold of the corresponding column sum. Average this rescaled matrix with its transpose to obtain a symmetric and positive definite matrix. 
\item Multiple each entry by $\sigma^2$, which is set to be $3$.
\end{enumerate}
\end{enumerate}

For each model, we consider three cases with different values for $p$: 
$$1)\   p=50; \quad 2)\ p=100;  \quad 3) \ p=200.$$
So, we consider a total of $12$ simulation settings. In each setting, $n = 100$ observations are generated, and results are aggregated based on $50$ replications. 

For estimation accuracy of $\Theta^0$, we use Frobenius norm  (denoted as Fnorm).  For selection accuracy, we consider three criteria: sensitivity, specificity and MCC (Matthews correlation coefficient):
$$\text{Specificity}=\frac{\text{TN}}{\text{TN+FP}}, \qquad \text{Sensitivity}=\frac{\text{TP}}{\text{TP+FN}},~~ \text{and}$$
$$\text{MCC}=\frac{\text{TP}\times \text{TN-FP}\times \text{FN}}{\sqrt{\text{(TP + FP)(TP + FN)(TN + FP)(TN + FN)}}},$$
where TP (true positive), FP (false positive), TN (true negative), and FN (false negative) are based on detection of edges in the graph corresponding to the true precision matrix $\Theta^0$. { MCC returns a value between $-1$ and $+1$, and the higher the MCC, the better the structure recovery is. A coefficient of $+1$ in MCC represents a perfect structure recovery, and we note that recovering all the edges simultaneously is very challenging and none of the existing methods are able to ensure that. In addition, we note} that it may not be meaningful to compare the results across graphs with different values of $p$ because the level of sparsity changes with $p$ which makes it difficult to assess the difficulty of the setting based on $p$ alone. For instance, for most models considered in our simulation study, the level of sparsity increases along with $p$, because of which all the methods have their specificity increasing when $p$ gets larger (see Tables \ref{model1}-\ref{model5}). So we recommend against comparing the results as $p$ changes and instead to compare the results across different methods within the same setting.

In the simulation study, we compare our method, denoted as BAGUS, with the following alternatives: GLasso from \cite{friedman2008sparse}, SPACE from \cite{peng2009partial} and CLIME from \cite{cai2011constrained}. They are all shown to have estimation consistency  under various conditions as discussed in Section \ref{sec:compare}. We also considered the regression based method from \cite{meinshausen2006high}, but the results are not presented here  because  tuning the parameters as suggested in \cite{meinshausen2006high} gave us ``NA'' for MCC in multiple scenarios considered here. 

For each simulated data set, tuning for our model uses the aforementioned BIC criterion with a parameter set of $\eta=0.5$, 
{$v_0=\tau=(0.4,2,4,20)\times\sqrt{\frac{1}{n\log p}}$ }
and $v_1$ ranges from $v_0\times (1.5, 3, 5, 10)$. The tuning parameters for GLasso  are chosen with 10-fold CV, the tuning parameters for SPACE are chosen from the BIC-like criterion proposed in \citet{peng2009partial} and the tuning and estimation for CLIME estimator is done using the R package \texttt{flare} \citep{flare} as suggested on the homepage\footnote{\url{http://www-stat.wharton.upenn.edu/~tcai/paper/html/Precision-Matrix.html}} of \cite{cai2011constrained}. For cross validation, the number of $\lambda$ values is set to be 40. Results for all the simulated cases are summarized in Tables \ref{model1}-\ref{model5}.

In almost all the settings considered, our method BAGUS performs the best in terms of both selection accuracy, i.e., MCC, and estimation accuracy, i.e., Fnorm. We believe that it is due to the adaptive nature of the Bayesian penalization and the weaker conditions under which the consistency results hold true for BAGUS. Other than BAGUS, SPACE usually performs well in terms of sparse selection and GLasso performs well in terms of estimation accuracy. However, SPACE has a large estimation error in most cases and GLasso tends to have smaller MCC. In our simulation study, CLIME estimator did not perform very well. It is particularly worth noting that for the star graph, where the assumption for CLIME fails (see discussion in Section \ref{sec:compare}), the performance of CLIME is particularly worse.


In Figure  \ref{fig:ROC}, we plot the receiver operating characteristic (ROC) curves for all the methods considered under different models by varying hyper (tuning) parameters for the case with $p=50$. This is to see the performance of different methods by removing the effect of tuning. Our method BAGUS  remains at the top in all the settings considered in terms of area under the ROC curve (AUC). This plot suggests that except for the star graph, performance of CLIME is not as poor as indicated by the selected graph, which suggests that the performance of CLIME could be improved by better tuning. However, for the star graph, CLIME is still observed to be particularly worse even in view of the ROC curve.

We also recorded the average of the estimated structures from the 50 replicates and compare it with the truth to get a visual understanding of the performance of different methods, shown in Figures 
\ref{fig:star}-\ref{fig:random}. 
 It is noticeable that GLasso and CLIME provide noisier estimates than BAGUS by including many zero entries in the selection; BAGUS and SPACE are sparser and appear closer to the true precision matrix. However, SPACE usually produces noisier estimates than BAGUS (for Models 1-3) and misses a lot of true signals for Model 4. In summary, BAGUS provides a highly competitive performance across the models considered.

\begin{sidewaysfigure}[!htbp]
\caption{Average of the estimated precision matrices for the model with the {\bf star structure} }\vspace{2ex} 
\includegraphics[width=1\linewidth]{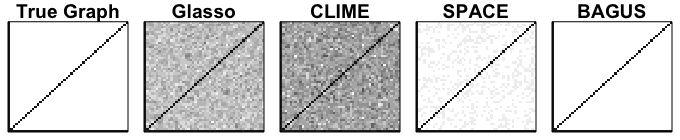}\vspace{4ex} 

\label{fig:star}


\caption{Average of the estimated precision matrices for the model with the {\bf AR(2) structure}}\vspace{2ex} 
\includegraphics[width=1\linewidth]{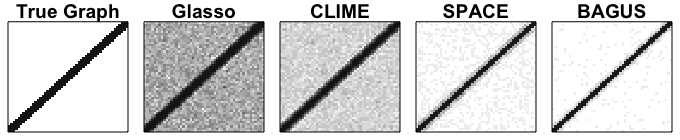}
\label{fig:AR2}
\end{sidewaysfigure}

\begin{sidewaysfigure}[!htbp]
\caption{Average of the estimated precision matrices for the model with the {\bf circle structure}}\vspace{2ex} 
\includegraphics[width=1\linewidth]{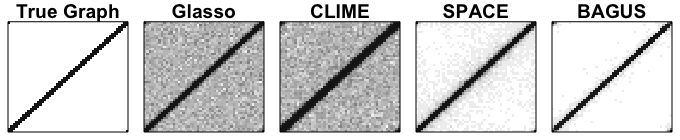}\vspace{4ex} 
\label{fig:circle}

\caption{Average of the estimated precision matrices for the model with the {\bf random structure}}\vspace{2ex} 
\includegraphics[width=1\linewidth]{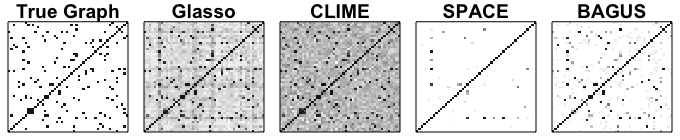}
\label{fig:random}
\end{sidewaysfigure}

	\begin{figure}[!htbp]

			\includegraphics[width=\linewidth]{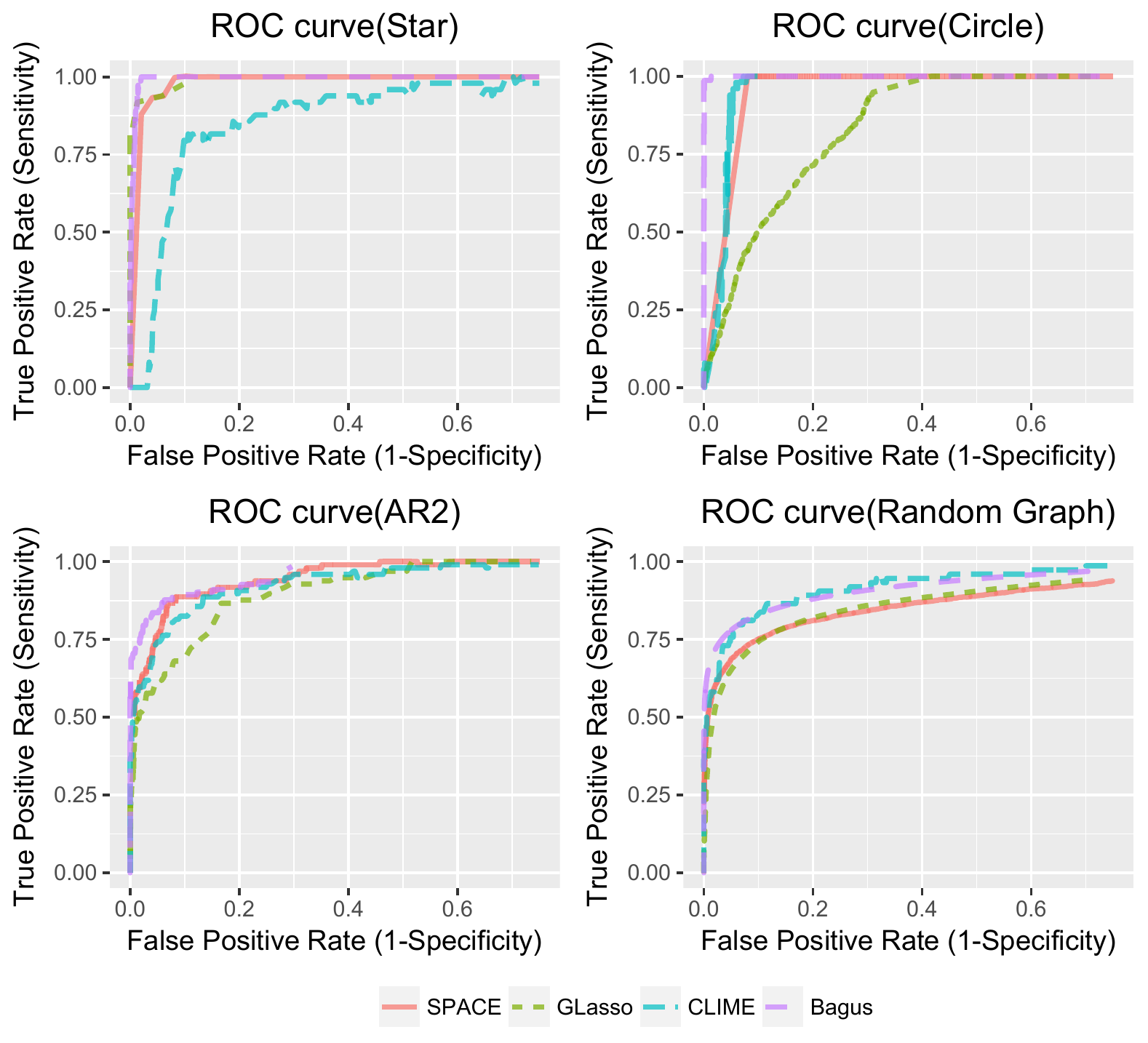}
\caption{ROC Curves for different methods and different data generating models with $p=50.$}

		\label{fig:ROC}
	\end{figure}

\subsection{Real Application: Telephone Call Center Data}
We now apply our method to the analysis of data from a telephone call center in a major U.S. northeastern financial organization. The data consists of the arrival time of each phone call in 2002 every day from 7 AM till midnight, except for six days when the data collecting machine is out of order. More details about this data can be found in \citet{shen2005analysis}.

Following the pre-processing as suggested by \citet{huang2006covariance} and \cite{fan2009network} for this data set, we divide each day into $102$ $10$-minute intervals and count the number of call arrivals for each interval, denoted as $N_{it}$ where $t=1:102$ and $i=1:239$. Only 239 days of data are considered here, after we remove  holidays and days when the data collecting machine was faulty. Represent the observations on the $i$-th day  as $Y_i = (Y_{i1}, Y_{i2}, \dots)^T$, a $102 \times 1$ vector with $Y_{it}=\sqrt{N_{it}+\frac{1}{4}}$, a variance stabilizing transformation of the number of calls. Let $\mu$ and $\Theta$ denote the mean vector and precision matrix of the 102-dimensional vector $Y$. 

We apply all the methods considered on the first 205 days of data to estimate $\Theta$, as well as $\mu$, and use the remaining 34 days of data to evaluate the performance. The performance evaluation  is carried out as follows. First divide the $102$ observations for each day into two parts $(Z_{i1}$ and $Z_{i2})$, where $Z_{i1}$ is a $51\times 1$ vector containing data from the first $51$ intervals on the $i$-th day and $Z_{i2}$ is also a $51\times 1$ vector containing the remaining $51$ observations, then partition the mean vector $\mu$ and the precision matrix $\Theta$ accordingly. Under the multivariate Gaussian assumption, the best mean squared error forecast of $Z_{i2}$ given $Z_{i1}$ is given by
\begin{equation} \label{eq:blup}
\mathbb{E}(Z_{i2} | Z_{i1})=u_{2}-\Theta_{22}^{-1}\Theta_{21}(Z_{i1}-u_1),
\end{equation}
which is also the best linear unbiased predictor for non-Gaussian data. So plugging the estimates of  $\mu$ and $\Theta$ based on the first 205 days into (\ref{eq:blup}), we evaluate the prediction accuracy for $Z_{i2}$ for the remaining 34 days. We adopt the same 
criterion used by \citet{fan2009network}, the average absolute forecast error (AAFE), to measure the prediction performance:
\begin{equation}
\textsf{AAFE}_t=\frac{1}{34}\sum_{i=206}^{239}|\hat{Y}_{it}-Y_{it}|.
\end{equation}
where $\hat{Y}_{it}$ and $Y_{it}$ denote the predicted and observed values, respectively. 

We compare the prediction performance based on estimates from our method BAGUS, the inverse of the sample covariance matrix (denoted as ``Sample"), GLasso and CLIME. The prediction errors for these methods at all $51$ time points are shown in Figure \ref{fig:pred}. Their average AAFE values are displayed in Table \ref{tab:pred}, along with the average AAFE values for Adaptive Lasso and SCAD taken from \cite{fan2009network}. 

\begin{figure}
\begin{center}
\includegraphics[width=0.9\linewidth]{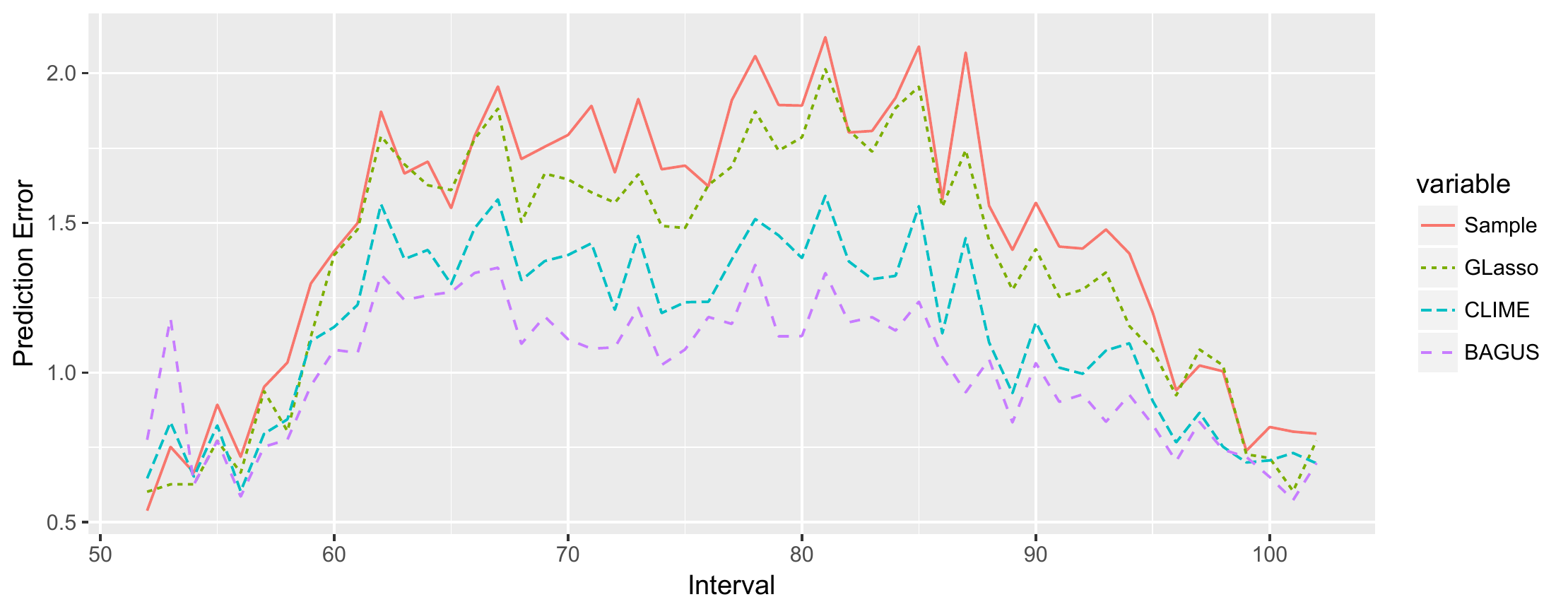}
\caption{Prediction error for the call center cata: $\text{AAFE}_t$ on $Y$ axis and $t$ on X axis.  }\label{fig:pred}
\end{center}
\end{figure}

\begin{table}[!htbp]
\begin{center}
\caption{Average Prediction error for different methods} \label{tab:pred}
  \resizebox{0.8\textwidth}{!}{\begin{minipage}{\textwidth}
\begin{tabular}{lllllllll}
\hline
&Sample& GLasso  & Adaptive Lasso & SCAD & CLIME&BAGUS \\ 
Average AAFE&1.46&1.38&1.34&1.31&1.14&{\bf1.00}\\ \hline
\end{tabular}
 \end{minipage}}
 \end{center}
\end{table}

From the results, we see that BAGUS and CLIME have a significantly improved performance in prediction accuracy when compared with the other methods. To look further into the estimates provided by these methods, we present  the sparsity structures estimated from GLasso, CLIME, and BAGUS in Figure \ref{fig:heat}. In this figure, yellow points (appear in light tone when converted to grayscale) indicate signals and blue points (dark tone in grayscale) indicate noise. In the Gaussian graphical model context, a yellow point suggests that the call arrivals in the corresponding two time intervals are conditionally dependent. It is interesting to find that a strong autoregressive type of dependence structure is present in estimators from all methods. However, the methods differ in terms of the degree of autoregression suggested by their corresponding estimates. The estimated structure from {BAGUS} is the most sparse one and suggests a small degree of autoregression compared to those of GLasso and CLIME. That is, BAGUS indicates that the telephone call arrivals majorly depend only on recent history, while others indicate dependence over a long history. Based on the prediction accuracies of different methods, the sparser dependence structure suggested by {BAGUS} seems sufficient to provide good prediction although it is difficult to know which structure, in reality, is closer to the underlying precision matrix. In terms of practical utility, this provides support in favor of storing and managing less amount of historical data that could potentially reduce cost of data management.

 \begin{landscape}
 	\begin{multicols*}{2}
 		\begin{table}[H]
 			\resizebox{0.54\linewidth}{!}{\begin{minipage}{\linewidth}
 					\caption{Model1 Star}
 					\label{model1}
 					\begin{tabular}{llllllllll}
 						\hline
 						\multicolumn{5}{c}{$n=100,p=50$}                  \\ \hline
 						&Fnorm& Specificity  & Sensitivity& MCC \\ 
 						GLasso     &   2.301(0.126) &0.687(0.015)& 0.998(0.004)& 0.339(0.011)\\ 
 						CLIME &  3.387(0.401)  &0.452(0.051)& 0.971(0.023)& 0.168(0.021) \\                   
 						SPACE    &   2.978(0.244)    &0.972(0.039)& 1.000(0.003)& 0.824(0.163)\\ 
 						BAGUS&\textbf{1.053(0.107)}&1.000(0.000)& 1.000(0.000)&\textbf{1.000(0.000)}\\
 						\hline
 						\multicolumn{5}{c}{$n=100,p=100$}                   \\ \hline
 						&Fnorm& Specificity  & Sensitivity& MCC  \\ 
 						GLasso        &4.219(0.118)&0.715(0.007)& 0.989(0.008)& 0.260(0.005)\\ 
 						CLIME&4.818(0.449)& 0.998(0.004)& 0.336(0.000)& 0.131(0.067)\\                   
 						SPACE     &3.207(0.311)&0.987(0.022)& 0.996(0.024)& 0.842(0.162)\\ 
 						BAGUS&\textbf{1.499(0.138)}&1.000(0.000)& 1.000(0.000)&\textbf{1.000(0.000)}\\
 						\hline
 						\multicolumn{5}{c}{$n=100,p=200$}                   \\ \hline
 						&Fnorm& Specificity  & Sensitivity& MCC  \\ 
 						GLasso        &3.028(0.068)&0.947(0.003)& 0.999(0.002)& 0.389(0.009)\\ 
 						CLIME&5.595(0.528)&0.978(0.018)&  0.000(0.000)& -0.014(0.006)\\                   
 						SPACE     &3.735(0.294)&0.985(0.007)& 1.000(0.000)& 0.656(0.138)\\ 
 						BAGUS &\textbf{2.006(0.100)}&1.000(0.000)& 1.000(0.001)& \textbf{1.000(0.001)}
\\
 						\hline\\
 					\end{tabular}

  \caption{Model 2: $AR(2)$ \label{model3}}
\begin{tabular}{lllllllll}
\hline
\multicolumn{5}{c}{$n=100,p=50$}                    \\ \hline
                   &Fnorm& Specificity  & Sensitivity& MCC \\ 
GLasso             &{\bf3.361(0.240)} &  0.479(0.056)&0.981(0.015)&0.251(0.028) \\ 
 CLIME&3.758(0.381)&0.822(0.054)& 0.906(0.039)&0.472(0.053)\\
SPACE               & 5.903(0.070)   &  0.982(0.004)&0.608(0.038)&0.656(0.029) \\ 
BAGUS&3.671(0.291)&0.997(0.002)& 0.551(0.032)& \textbf{0.707(0.025)}\\
\hline
 \multicolumn{5}{c}{$n=100,p=100$}                   \\ \hline
                   &Fnorm& Specificity  & Sensitivity& MCC  \\ 
GLasso            & 8.130(0.035)& 0.901(0.007)& 0.745(0.028)&0.382(0.017)\\ 
 CLIME&5.595(1.578)&0.837(0.075)& 0.821(0.191)& 0.371(0.085)\\
SPACE               &9.819(0.083)&0.991(0.002)&0.566(0.025)&0.625(0.021)\\ 
BAGUS&\textbf{5.330(0.369)}& 0.998(0.001)& 0.549(0.018)& \textbf{0.707(0.022)}\\
\hline
                     \multicolumn{5}{c}{$n=100,p=200$}                   \\ \hline
                  &Fnorm& Specificity  & Sensitivity& MCC  \\ 
GLasso        & 11.728(0.045)&0.990(0.001)& 0.478(0.017)& 0.481(0.014)\\ 
 CLIME&11.552(0.382)&0.989(0.004)& 0.580(0.031)& 0.539(0.028)\\                   
SPACE     &13.696(0.079)&0.995(0.000)& 0.518(0.018)& 0.588(0.013)\\ 
BAGUS&\textbf{8.214(0.548)}&0.998(0.001)& 0.543(0.015)& \textbf{0.677(0.027)}\\
\hline
\end{tabular}
 \end{minipage}}
\end{table}

\begin{table}[H]
  \resizebox{0.54\linewidth}{!}{\begin{minipage}{\linewidth}
\centering
  \caption{Model 3: Circle \label{model4}}
\begin{tabular}{lllllllll}
\hline
\multicolumn{5}{c}{$n=100,p=50$}                   \\ \hline
                   &Fnorm& Specificity  & Sensitivity&   MCC\\ 
GLasso              &\textbf{4.319(0.174)}&  0.492(0.064)& 1.000(0.000)&0.196(0.024)   \\ 
 CLIME&5.785(0.440)&0.555(0.026)& 1.000(0.000)& 0.221(0.010)\\
SPACE                   &19.402(0.232)&  0.930(0.006)&1.000(0.000)&0.595(0.019)\\ 
BAGUS&\textbf{4.253(0.578)}&0.993(0.004)& 0.964(0.029)& \textbf{0.903(0.049)}\\
\hline
 \multicolumn{5}{c}{$n=100,p=100$}                   \\ \hline
                    &Fnorm& Specificity  & Sensitivity&MCC\\ 
GLasso           &6.981(0.192) & 0.647(0.005)& 1.000(0.000)& 0.189(0.002) \\ 
 CLIME& 19.282(2.802)&0.224(0.226)& 0.995(0.015)& 0.069(0.058)\\
SPACE                &27.737(0.345)&0.975(0.010)& 0.994(0.008)&0.674(0.062)\\ 
BAGUS&\textbf{6.012(0.513)}&0.996(0.002)& 0.957(0.032)& \textbf{0.895(0.055)}\\
\hline
                     \multicolumn{5}{c}{$n=100,p=200$}                   \\ \hline
                  &Fnorm& Specificity  & Sensitivity& MCC  \\ 
GLasso        &\textbf{7.664(0.209)}&0.752(0.003)& 1.000(0.000)& 0.172(0.001)\\ 
 CLIME&33.009(0.535)&0.857(0.154)& 0.769(0.167)& 0.209(0.052)\\                   
SPACE     &32.142(0.832)&0.981(0.012)& 0.783(0.212)& 0.485(0.129)\\ 
BAGUS &10.378(1.001)&0.995(0.001)& 0.886(0.033)& \textbf{0.752(0.028)}\\
\hline\\
\end{tabular}

  \caption{Model4: Random Graph \label{model5}}
\begin{tabular}{lllllllll}
\hline
 \multicolumn{5}{c}{$n=100,p=50$}                     \\ \hline
                   &Fnorm& Specificity  & Sensitivity& MCC \\ 
GLasso            & 7.017(0.256)&0.877(0.010)& 0.766(0.039)& 0.417(0.027)\\
CLIME&11.347(0.452)&0.971(0.012)& 0.614(0.068)& 0.572(0.042)\\
SPACE               &12.278(0.183)&1.000(0.000)& 0.073(0.031)& 0.257(0.051)\\
BAGUS&\textbf{5.811(0.357)}&0.999(0.001)& 0.443(0.032)& \textbf{0.637(0.027)}\\
\hline
\multicolumn{5}{c}{$n=100,p=100$}                   \\ \hline
                   &Fnorm& Specificity  & Sensitivity& MCC \\ 
GLasso            &  11.851(0.900)& 0.837(0.047)& 0.720(0.049)&0.285(0.033)\\
CLIME&12.649(1.587)&0.735(0.153)& 0.761(0.120)& 0.243(0.123)\\
SPACE               &17.706(0.203)& 1.000(0.000)&0.068(0.015)& 0.236(0.028)\\
BAGUS&\textbf{8.754(0.366)}&0.999(0.001)& 0.400(0.022)& \textbf{0.598(0.022)}\\
\hline
                     \multicolumn{5}{c}{$n=100,p=200$}                   \\ \hline
                  &Fnorm& Specificity  & Sensitivity& MCC  \\ 
GLasso        & 15.054(0.356)&0.951(0.012)& 0.633(0.029)& 0.307(0.017)\\ 
 CLIME&23.568(0.954)&0.993(0.004)& 0.469(0.048)&0.492(0.038)\\                   
SPACE     &24.997(0.213)&0.999(0.000)& 0.090(0.014)& 0.221(0.024)\\ 
BAGUS&\textbf{13.096(0.522)}& 0.999(0.000)& 0.382(0.050)& \textbf{0.565(0.032)}\\
\hline
\end{tabular}
 \end{minipage}}
\end{table}

\end{multicols*}
\end{landscape}

\begin{figure}[!htbp]
\begin{center}
\includegraphics[width=1\linewidth]{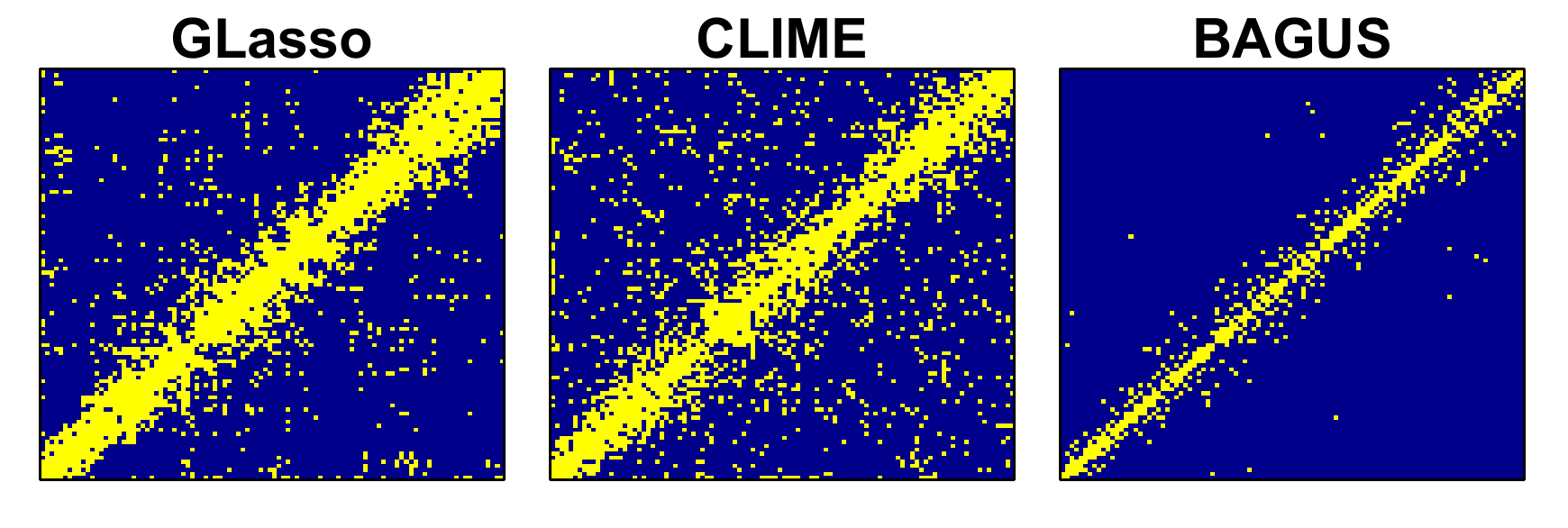}
\end{center}
\caption{Sparsity structures estimated for different methods for the call center data}\label{fig:heat}
\end{figure}

\section{Conclusion}
In high dimensional data analysis, there is a large literature on penalization from a frequentist viewpoint majorly focusing on Lasso based convex penalties and some non-convex penalties such as SCAD. On the other hand, in the Bayesian framework, a variety of shrinkage and sparsity inducing prior distributions have been proposed. In the context of graphical models, our work demonstrates that spike-and-slab priors with Laplace distributions provide adaptive penalization that leads to better theoretical and empirical performance compared to state-of-the-art methods. Since some recent papers \citep{rovckova2016fast, deshpande2017simultaneous} have also found spike-and-slab Lasso priors to be useful in other high dimensional contexts, we believe that our strategy of Bayesian regularization will be advantageous in a broad range of high dimensional problems and that its success demonstrated in our work will motivate further interest in this direction.

\newpage
\begin{center}
{\large\bf SUPPLEMENTARY MATERIAL}
\end{center}

\spacingset{1.48}  
\section*{Appendix A: Proofs of the Main Theorems}\label{sec:proof}

For convenience, we introduce the following additional notation that will be used throughout the Appendix. 
\begin{enumerate}[i.]
\item Let $\tilde{W}$ denote the difference between the sample covariance matrix $S$ and  the true covariance matrix $\Sigma^0 = \left(\Theta^{0}\right)^{-1}$ and $\Delta$  the difference between an estimate $\tilde{\Theta}$ and the true precision matrix $\Theta^{0}$. That is,  
\begin{eqnarray*}
\tilde{W} &=& S- \Sigma^{0} \\ 
\Delta &=& \tilde{\Theta}- \Theta^{0}. 
\end{eqnarray*}
\item Let $R\left(\Delta\right)$ denote the difference between $n\tilde{\Theta}^{-1}/2$, the gradient of $n\log\det(\tilde{\Theta})/2$, and its first-order Taylor expansion at $\Theta^{0}$: 
\begin{equation*} R\left(\Delta\right) = \frac{n}{2}\left(\tilde{\Theta}^{-1}-\Sigma^{0}+\Sigma^{0}\Delta\Sigma^{0}\right). 
\end{equation*}

\item Recall our objective function 
\[ 
L(\Theta) = \frac{n}{2} \Big ( \text{tr}(S\Theta) - \log \det(\Theta) \Big ) + \frac{1}{2} \sum_{i, j} \text{pen}_{SS}(\theta_{ij}) +\sum_{i} \text{pen}_1(\theta_{ii}), \]
where
 \begin{eqnarray*}
 \text{pen}_{SS}(\theta_{ij})  =   
- \log \Big [ \Big(\frac{\eta}{2v_1}\Big )e^{-\frac{|\theta_{ij}|}{v_1}} +\Big(\frac{1-\eta}{2v_0} \Big ) e^{-\frac{|\theta_{ij}|}{v_0}}\Big ] , ~ \text{and }~  \text{pen}_{1}(\theta_{ii}) =  \tau | \theta_{ii}|
\end{eqnarray*}
 denote the penalty terms on $\theta_{ij}$ $(i \ne j)$ and $\theta_{ii}$, respectively. 
 
Let $Z_{ij}$ denote the subgradient of the penalty term with respect to $\theta_{ij}$:
 \[Z_{ij}= Z_{ij}(\theta_{ij}) = \begin{cases} 
      \tau & \quad \text{if } i=j \\
      \frac{1}{2}\text{pen}{'}_{SS}(\theta_{ij}) & \quad \text{if } i\ne j,\quad \theta_{ij}\ne0 \\
      \left[-1,1\right]\times \frac{\frac{\eta}{2v_1^2}+\frac{1-\eta}{2v_0^2}}{\frac{\eta}{v_1}+\frac{1-\eta}{v_0}} &  \quad \text{if } i\ne j,\quad \theta_{ij}=0
   \end{cases}
\]
where 
\[
\text{pen}{'}_{SS}\left(\theta_{ij}\right)=\frac{\frac{\eta}{2v_1^2}e^{-\frac{|\theta_{ij}|}{v_1}}+\frac{1-\eta}{2v_0^2}e^{-\frac{|\theta_{ij}|}{v_0}}}{\frac{\eta}{2v_1}e^{-\frac{|\theta_{ij}|}{v_1}}+\frac{1-\eta}{2v_0}e^{-\frac{|\theta_{ij}|}{v_0}}}\text{sign}\left(\theta_{ij}\right). 
\]
Let $Z=[Z_{ij}]$, then the  subgradient of the objective function $L(\Theta)$ is
\begin{equation*}
\partial L(\Theta) = \frac{n}{2}\left(S-\Theta^{-1}\right)+Z.
\end{equation*}
\item We denote the index set of diagonal entries as $\mathcal{D}:= \{(i,j):i=j\}$.
{For any subset {$\mathcal{S}$} of $\{(i,j):1\le i,j\le p\}$ and $p \times p$ matrix $A$, we use $A_{\mathcal{S}}$ to denote the submatrix of $A$ with entries indexed by {$\mathcal{S}$}.}
\end{enumerate}
\subsubsection*{}
In this Appendix, we first prove the following main result. 

\begin{thm}
\label{thm:proof}

Assume condition (A1) and $\|\tilde{W}\|_{\infty}=\max_{ij}|s_{ij}-\sigma^0_{ij}|\le C_1\sqrt{\log p/n}$. If \\
(i) the prior hyper-parameters $v_0,v_1, \eta$ and $\tau$ satisfy:
\begin{equation}
\begin{cases}
\frac{1}{nv_1}={C_3}\sqrt{\frac{\log p}{n}}(1-\varepsilon_1),\text{ where } C_3<C_2 , \varepsilon_1>0,\\
\frac{1}{nv_0}>C_4\sqrt{\frac{\log p}{n}},\\
\frac{v_1^2(1-\eta)}{v_0^2\eta}\le \varepsilon_1 p^{2(C_2-C_3)M_{\Gamma^0}[C_4-C_3]},\\
\tau\le C_3\frac{n}{2}\sqrt{\frac{\log p}{n}},
\end{cases}
\end{equation}
where $C_4=(C_1+M_{\Sigma^0}^22(C_1+C_3)M_{\Gamma^0}+
6(C_1+C_3)^2dM_{\Gamma^0}^2M_{\Sigma^0}^3/M)$, \\
(ii) the spectral norm $B$ satisfies $1/k_1+2d(C_1+C_3)M_{\Gamma^0}\sqrt{\log p/n}<B<(2nv_0)^{\frac{1}{2}}$, and \\
(iii) the sample size $n$ satisfies $\sqrt{n}\ge M\sqrt{\log p}$, 
 where 
 $$M=\max\Big\{2d(C_1+C_3)M_{\Gamma^0}{\max\Big({3M_{\Sigma^0}},{3M_{\Gamma^0}{M^3_{\Sigma^{0}}}},2/k_1^2\Big)},2C_3\varepsilon_1/k_1^2\Big\},$$ then the MAP estimator $\tilde{\Theta}$ satisfies
\begin{equation*}
\|\tilde{\Theta}-\Theta^0\|_{\infty}<2(C_1+C_3)M_{\Gamma^0}\sqrt{\frac{\log p}{n}}.
\end{equation*}
\end{thm}

\medskip
Before presenting our proof, we list two preliminary results as lemmas and list some properties of the penalty function $\text{pen}_{SS}(\delta)$, which will be useful. {Proofs of these lemmas are in Appendix B.}

\begin{lemma}
\label{lemma:4}
Define $r:=\max\left\{2M_{\Gamma^0} \Big(\|\tilde{W}\|_{\infty}+\frac{2}{n}\max (\frac{1}{2}\text{pen}^{'}_{SS}(\delta),\tau)\Big),2(C_1+C_3)M_{\Gamma^0}\sqrt{\frac{\log p}{n}}\right\}$, and $\mathcal{A}:=\left \{ \Theta: \frac{n}{2}\left(S-\Theta^{-1}\right)_{\mathcal{B}}+Z_{\mathcal{B}} =0, \Theta\succ 0,  \| \Theta\|_2\le B \right \}$ with $\mathcal{B}=\{(i,j):|\theta_{ij}^0|>2(C_1+C_3)M_{\Gamma^0}\sqrt{\log p/n}\}\cup \mathcal{D}$.
If parameters $r$ and $B$ satisfy:
\begin{equation*}
\begin{cases}
r \le \min \left\{\frac{1}{3M_{\Sigma^0}d},\frac{1}{3dM_{\Gamma^0}M_{\Sigma^0}^3} \right\},\\
\min{|\theta_{\mathcal{B}\cap \mathcal{D}^c}^{0}|}\ge r+\delta,\\
1/k_1+dr<B,
\end{cases}
\end{equation*}
for some $\delta >0$, {where $k_1$ is the lower bound on $\lambda_{\min}(\Sigma^0)$,}
then the set $\mathcal{A}$ is non-empty. Moreover, there exists a $\tilde{\Theta}\in \mathcal{A}$ such that  
$\|\Delta\|_{\infty} :=\|\tilde{\Theta}-\Theta^0\|_{\infty}\le r.$
\end{lemma}


\begin{lemma}
\label{lemma:5}
Suppose that $\|\tilde{\Theta}-\Theta^0\|_{\infty}\le r$,
then 
\begin{eqnarray}
&&\|\tilde{\Theta}-\Theta^0\|_{F}\le r\sqrt{p+s}, \label{eq1:lemma5}\\
&&\vertiii{\tilde{\Theta}-\Theta^0}_{\infty}, \|\tilde{\Theta}-\Theta^0\|_{2}\le r\min\{d,\sqrt{p+s}\}, \text{ and} \label{eq2:lemma5}\\
&&\| \tilde{\Theta}^{-1}-\Sigma^0 \|_{\infty}\le M_{\Sigma^0}^2r+\frac{3}{2}dM^3_{\Sigma^{0}}r^2.
\label{eq3:lemma5}
\end{eqnarray}
\end{lemma}

\noindent \textbf{Properties of $\text{pen}_{SS}(\delta)$} \vspace{1ex}\\
We now provide some useful results on the penalty function  $\text{pen}_{SS}(\delta)$. 
\begin{itemize}
\item Bound on the magnitude of the first derivative of $\text{pen}_{SS}(\delta)$:  
\begin{eqnarray}
\frac{1}{n}|\text{pen}{'}_{SS}(\delta)|&=\frac{\frac{\eta}{2v_1^2}e^{-\frac{|\delta|}{v_1}}+\frac{1-\eta}{2v_0^2}e^{-\frac{|\delta|}{v_0}}}{n\left(\frac{\eta}{2v_1}e^{-\frac{|\delta|}{v_1}}+\frac{1-\eta}{2v_0}e^{-\frac{|\delta|}{v_0}}\right)} \nonumber \\
&=\frac{1}{nv_1}+\frac{\frac{1}{n}(\frac{1}{v_0}-\frac{1}{v_1})}{\frac{\eta v_0}{(1-\eta)v_1}e^{\frac{|\delta|}{v_0}-\frac{|\delta|}{v_1}}+1} \nonumber \\
&<\frac{1}{nv_1} \left(1+\frac{\frac{v_1^2(1-\eta)}{v_0^2\eta}}{e^{\frac{|\delta|}{v_0}-\frac{|\delta|}{v_1}}} \right).
\label{eq:firstderv}
\end{eqnarray}
Choose $1/\left(nv_0\right)>C_4\sqrt{\log p/n}$ and $1/\left(nv_1\right)<C_3\sqrt{\log p/n}$ as in Theorem \ref{thm:proof}, and if further let $v_1^2\left(1-\eta\right)/\left(v_0^2\eta\right)=\xi p^{\psi[C_4-C_3]}$, when $\delta\ge\psi\sqrt{\log p/n}$, then we have
\begin{equation}
\frac{\frac{v_1^2\left(1-\eta\right)}{v_0^2\eta}}{e^{\frac{|\delta|}{v_0}-\frac{|\delta|}{v_1}}}\le \frac{\xi p^{\psi[C_4-C_3]}} {p^{\psi[C_4-C_3]}}\le \xi.
\label{eq:reason}
\end{equation} 
Let $\xi$ to be sufficiently small, i.e.,  $\xi<\varepsilon_1$, then we have  \[ \frac{1}{n}|\text{pen}{'}_{SS}(\delta)| < C_3\sqrt{ \frac{\log p}{n}}.\]

\item Bound on the magnitude of the second derivative of $\text{pen}_{SS}(\delta)$: \vspace{1ex}\\
 With the same choice of  $v_0$ and $v_1$ as in Theorem \ref{thm:proof}, when $\delta\ge\psi\sqrt{\log p/n}$, we have
\begin{eqnarray}
\frac{1}{2n}|\text{pen}_{SS}^{''}(\delta)|& = & \frac{\left(\frac{1}{v_0}-\frac{1}{v_1}\right)\frac{\eta v_0}{(1-\eta)v_1}e^{\frac{\delta}{v_0}-\frac{\delta}{v_1}}}{2n \left(\frac{\eta v_0}{(1-\eta)v_1}e^{\frac{\delta}{v_0}-\frac{\delta}{v_1}}+1 \right)^2} \nonumber \\
&< &\frac{ \left(\frac{1}{v_0}-\frac{1}{v_1} \right)}{2n \left( \frac{\eta v_0}{(1-\eta)v_1}e^{\frac{\delta}{v_0}-\frac{\delta}{v_1}}+1 \right)} \nonumber \\
&< & \frac{(1-\eta)v_1}{2nv_0^2\eta e^{\frac{\delta}{v_0}-\frac{\delta}{v_1}}}<\frac{\xi}{2nv_1} \label{eq:secondderv} \\
&< & \frac{C_3}{2}\xi\sqrt{\frac{\log p}{n}}< \frac{C_3}{2}\varepsilon_1\sqrt{\frac{\log p}{n}}, \label{eq:secondder:2}
\end{eqnarray}
where (\ref{eq:secondderv}) is due to (\ref{eq:reason}). In addition, when $n$ satisfies the condition $(iii)$ in Theorem \ref{thm:proof}, (\ref{eq:secondder:2}) is always upper bounded by $\frac{1}{4}k_1^2$.
\end{itemize}

\begin{proof}[\textbf{Proof of Theorem \ref{thm:proof}}] 
Our proof is inspired by the techniques from \cite{rothman2008sparse} 
and \cite{ravikumar2011high}. 

Here is the outline of the proof. 
\begin{itemize}
\item \emph{Step 1:} Construct a solution set $\mathcal{A}$ for the constraint problem:
$$
\arg\min_{\Theta\succ 0,\|\Theta\|_2\le B,\Theta_{\mathcal{B}^c}=0}L\left(\Theta\right),$$ by defining
\begin{equation*}
\mathcal{A} =\left \{ \Theta: \frac{n}{2}\left(S-\Theta^{-1}\right)_{\mathcal{B}}+Z_{\mathcal{B}} =0, \Theta\succ 0,  \| \Theta\|_2\le B \right \}, 
\end{equation*}
where $\mathcal{B}= \{(i,j):|\theta_{ij}^0|>2(C_1+C_3)M_{\Gamma^0}\sqrt{\log p/n}\}\cup \mathcal{D}$. 
{For $\theta^0_{ij}\in \mathcal{B}\cap \mathcal{D}^c$ and , define $\min\left(|\theta_{ij}^0|\right)$ as $2(C_1+C_2)M_{\Gamma^0}\sqrt{\log p/n}.$ We then have $|\theta_{ij}^0|\ge2(C_1+C_2)M_{\Gamma^0}\sqrt{\log p/n}$ when $\theta^0_{ij}\in \mathcal{B}\cap \mathcal{D}^c$ and $|\theta_{ij}^0|\le2(C_1+C_3)M_{\Gamma^0}\sqrt{\log p/n}$ when $\theta^0_{ij}\in \mathcal{B}^c\cap\mathcal{D}^c.$}

\item \emph{Step 2:}  Prove $\mathcal{A}$ is not empty and further show that there exists $\tilde{\Theta} \in \mathcal{A}$ satisifying
$\|\tilde{\Theta}-\Theta^0\|_{\infty}=O_p \Big (\sqrt{\log p/n} \Big ).$

\item \emph{Step 3:} Finally prove that $\tilde{\Theta}$, which is positive definite by construction, is a local minimizer of the loss function $L(\Theta)$ by showing $L(\Theta) \ge L(\tilde{\Theta})$ for any $\Theta$ in a small neighborhood of $\tilde{\Theta}$. Since $L(\Theta)$ is strictly convex when $B<(2nv_0)^{\frac{1}{2}}$, we then conclude that $\tilde{\Theta}$ is the unique minimizer such that $\|\tilde{\Theta}-\Theta^0\|_\infty=O_p\left(\sqrt{\log p/n}\right)$.
\end{itemize}

At \emph{Step 2}, we apply Lemma \ref{lemma:4}. First we check   its  conditions. 
\begin{enumerate}
\item Consider $r=2(C_1+C_3)M_{\Gamma^0}\sqrt{\log p/n}$. For ${\theta^0_{ij}}\in \mathcal{B}\cap\mathcal{D}^c$, we have ${\theta^0_{ij}}\ge r+2(C_2-C_3)M_{\Gamma^0}\sqrt{\log p/n}$. That is, the $\delta$ defined in Lemma \ref{lemma:4} is greater or equal to $2(C_2-C_3)M_{\Gamma^0}\sqrt{\log p/n}$.

\item Recall the properties of $\text{pen}_{SS}(\delta)$. 
We have $|\text{pen}{'}_{SS}(\delta)|/n<C_3\sqrt{\log p/n}$. With the bound of $\|\tilde{W}\|_\infty$ and the condition on sample size $n$, we have
\begin{eqnarray*}
2M_{\Gamma^0} \left(\|W\|_{\infty}+\max \left(\frac{1}{n} \text{pen}^{'}_{SS}(\delta),\frac{2}{n}\tau\right)\right)&\le& 2(C_1+C_3)M_{\Gamma^0}\sqrt{\frac{\log p}{n}}\\
&\le & \min \left \{\frac{1}{3M_{\Sigma^0}d},\frac{1}{\frac{3}{2}dM_{\Gamma^0}M_{\Sigma^0}^3} \right\}.
\end{eqnarray*}
\end{enumerate}

Thus, conditions for Lemma \ref{lemma:4} are all satisfied. By Lemma \ref{lemma:4}, we conclude that there exists a solution $\tilde{\Theta} \in \mathcal{A}$ satisfying
\begin{equation*} 
\|\tilde{\Theta}-\Theta^0\|_{\infty}=\|\Delta\|_{\infty}\le2(C_1+C_3)M_{\Gamma^0}\sqrt{\frac{\log p}{n}}. 
\end{equation*}
That is, the solution $\tilde{\Theta}$ we constructed is $O_p\left(\sqrt{\log p/n}\right)$ from the truth in entrywise $l_{\infty} $ norm.

At \emph{Step 3}, we need to show that the solution $\tilde{\Theta}$ we constructed is indeed a local minimizer of the objective function $L(\Theta)$. It suffices to show that 
$$G(\Delta_1)=L(\tilde{\Theta}+\Delta_1)-L(\tilde{\Theta}) \ge 0$$
for any  $\Delta_1$ with $\|\Delta_1\|_{\infty} \le \epsilon$. Re-organize $G(\Delta_1)$ as follows:
\begin{equation*}
\begin{split}
G(\Delta_1) =~ &\frac{n}{2}\Big(tr\left(\Delta_1\left(S-\tilde{\Theta}^{-1}\right)\right)-\left(\log|\tilde{\Theta}+\Delta_1|-\log|\tilde{\Theta}|\right)+tr\left(\Delta_1\tilde{\Theta}^{-1}\right)\Big)\\
&-\sum_{i<j}\log\left(\frac{\eta}{2v_1}e^{-\frac{|\tilde{\theta}_{ij}+{\Delta_1}_{ij}|}{v_1}}+\frac{1-\eta}{2v_0}e^{-\frac{|\tilde{\theta}_{ij}+{\Delta_1}_{ij}|}{v_0}}\right)\\
&+\sum_{i<j}\log\left(\frac{\eta}{2v_1}e^{-\frac{|\tilde{\theta}_{ij}|}{v_1}}+\frac{1-\eta}{2v_0}e^{-\frac{|\tilde{\theta}_{ij}|}{v_0}}\right)+\tau\sum_{i}\left(\tilde{\theta}_{ii}+{\Delta_1}-\tilde{\theta}_{ii}\right)\\
&=\text{(I)} + \text{(II)}+ \text{(III)},\\
\end{split}
\end{equation*}
where
\begin{eqnarray*}
\text{(I)}&=&\frac{n}{2}\Big(tr\left({\Delta_1}\left(S-\tilde{\Theta}^{-1}\right)\right)-\left(\log|\tilde{\Theta}+{\Delta_1}|-\log|\tilde{\Theta}|\right)+tr\left({\Delta_1}\tilde{\Theta}^{-1}\right)\Big),\\
\text{(II)}&=&-\frac{1}{2}\sum_{i<j}\log\left(\frac{\eta}{2v_1}e^{-\frac{|\tilde{\theta}_{ij}+{\Delta_1}_{ij}|}{v_1}}+\frac{1-\eta}{2v_0}e^{-\frac{|\tilde{\theta}_{ij}+{\Delta_1}_{ij}|}{v_0}}\right)\\
&& +\frac{1}{2}\sum_{i<j}\log\left(\frac{\eta}{2v_1}e^{-\frac{|\tilde{\theta}_{ij}|}{v_1}}+\frac{1-\eta}{2v_0}e^{-\frac{|\tilde{\theta}_{ij}|}{v_0}}\right),\\
\text{(III)}&=&\tau\sum_{i}\left(\tilde{\theta}_{ii}+{\Delta_1}_{ii}-\tilde{\theta}_{ii}\right)=\tau{\Delta_1}_{ii}.
\end{eqnarray*}

Bound (I) as follows. 
\begin{eqnarray*}
&&\log|\tilde{\Theta}+{\Delta_1}|-\log|\tilde{\Theta}| \\
&=&\text{tr} \left({\Delta_1}\tilde{\Theta}^{-1}\right)-\text{vec}\left(\Delta_1\right)^{T}\int_{0}^{1}(1-v)\Big((\tilde{\Theta}^{-1}+v{\Delta_1})^{-1}\otimes (\tilde{\Theta}^{-1}+v{\Delta_1})^{-1}dv\Big)\text{vec}({\Delta_1})  \\
&\le& \text{tr} \left({\Delta_1}\tilde{\Theta}^{-1}\right)-\frac{1}{4}k_1^2\|{\Delta_1}\|_F^2.
\end{eqnarray*}
where the last inequality can be shown with the same proof for Theorem 1 in \cite{rothman2008sparse} with $\sqrt{n}\ge 4(C_1+C_3)dM_{\Gamma^0}/k_1^2\sqrt{\log p}$. Thus,
\begin{eqnarray*}
\text{(I)}&\ge& \frac{n}{2}\Big(tr\left({\Delta_1}\left(S-\tilde{\Theta}^{-1}\right)\right)+\frac{1}{4}k_1^2\|{\Delta_1}\|_F^2\Big)\\
&= & \frac{n}{2}\Big(\sum_{i,j}\left({\Delta_1}_{ij}\left(s_{ij}-{\tilde{\Theta}^{-1}}_{ij}\right)\right)+\frac{1}{4}k_1^2\|{\Delta_1}\|_F^2\Big).
\end{eqnarray*}

Next consider $\text{(II)}$. For any $(i,j)$ $\notin$ $\mathcal{B}$,  $\tilde{\theta}_{ij}=0$, $|\tilde{\theta}_{ij}+{\Delta_1}_{ij}|=|{\Delta_1}_{ij}|$, and therefore
\begin{eqnarray*}
&&-\log\left(\frac{\eta}{2v_1}e^{-\frac{|\tilde{\theta}_{ij}+{\Delta_1}_{ij}|}{v_1}}+\frac{1-\eta}{2v_0}e^{-\frac{|\tilde{\theta}_{ij}+{\Delta_1}_{ij}|}{v_0}}\right)+\log\left(\frac{\eta}{2v_1}e^{-\frac{|\tilde{\theta}_{ij}|}{v_1}}+\frac{1-\eta}{2v_0}e^{-\frac{|\tilde{\theta}_{ij}|}{v_0}}\right)\\
&=&\log\frac{\Big(\frac{\eta}{2v_1}e^{-\frac{|0|}{v_1}}\Big)+\Big(\frac{1-\eta}{2v_0}e^{-\frac{|0|}{v_0}}\Big)}{\Big(\frac{\eta}{2v_1}e^{-\frac{|{\Delta_1}_{ij}|}{v_1}}\Big)+\Big(\frac{1-\eta}{2v_0}e^{-\frac{|{\Delta_1}_{ij}|}{v_0}}\Big)}\\
&=&\frac{|{\Delta_1}_{ij}|}{v_0}-\log\Big(\frac{v_0\eta e^{\frac{|{\Delta_1}_{ij}|}{v_0}-\frac{|{\Delta_1}_{ij}|}{v_1}}+v_1{(1-\eta)}}{v_0\eta+v_1{(1-\eta)}}\Big).
\end{eqnarray*}
For any $(i,j)$ $\in$ $\mathcal{B}$ and $i\ne j$, applying Taylor expansion, for some $v\in(0,1), $ we have
\begin{equation*}
\begin{split}
&-\log\left(\frac{\eta}{2v_1}e^{-\frac{|\tilde{\theta}_{ij}+{\Delta_1}_{ij}|}{v_1}}+\frac{1-\eta}{2v_0}e^{-\frac{|\tilde{\theta}_{ij}+{\Delta_1}_{ij}|}{v_0}}\right)+\log\left(\frac{\eta}{2v_1}e^{-\frac{|\tilde{\theta}_{ij}|}{v_1}}+\frac{1-\eta}{2v_0}e^{-\frac{|\tilde{\theta}_{ij}|}{v_0}}\right)\\
&=\text{pen}_{SS}{'}(\tilde{\theta}_{ij}){\Delta_1}_{ij}+\frac{1}{2}\text{pen}_{SS}{''}\left(\tilde{\theta}_{ij}+v{\Delta_1}_{ij}\right){\Delta_1}_{ij}^2. 
\end{split}
\end{equation*}

Combining the results above, we have
\begin{equation*}
\begin{split}
G(\Delta_1)\ge& \frac{n}{2}\Big(\sum_{i,j}\left({\Delta_1}_{ij}(s_{ij}-{\tilde{\Theta}^{-1}}_{ij})\right)+\frac{1}{4}k_1^2\|{\Delta_1}\|_F^2\Big)+\sum_{i=j,\in \mathcal{B}}\tau{\Delta_1}_{ii}\\
&+\frac{1}{2}\sum_{i\ne j, \in \mathcal{B}}\left(\text{pen}_{SS}^{'}(\tilde{\theta}_{ij}){\Delta_1}_{ij}+\frac{1}{2}\text{pen}_{SS}^{''}(\tilde{\theta}_{ij}+v{\Delta_1}_{ij}){\Delta_1}_{ij}^2\right)\\
&-\frac{1}{2}\sum_{\notin \mathcal{B}}\left(-\frac{|{\Delta_1}_{ij}|}{v_0}+\log\Big(\frac{v_0\eta e^{\frac{|{\Delta_1}_{ij}|}{v_0}-\frac{|{\Delta_1}_{ij}|}{v_1}}+v_1{(1-\eta)}}{v_0\eta+v_1{(1-\eta)}}\Big)\right)\\
&=\text{(A)}+\text{(B)}+\text{(C)},
\end{split}
\end{equation*}
where
\begin{equation*}
\begin{split}
\text{(A)}=&\frac{n}{2}\Big(\sum_{(i,j)\in \mathcal{B}}{\Delta_1}_{ij}(s_{ij}-{\tilde{\Theta}^{-1}}_{ij}+\frac{2}{n}Z_{ij})\Big),\\
\text{(B)}=&\frac{n}{2}\left(\sum_{(i,j)\notin \mathcal{B}}\left({\Delta_1}_{ij}(s_{ij}-{\tilde{\Theta}^{-1}}_{ij})-\frac{1}{n}\Big(-\frac{|{\Delta_1}_{ij}|}{v_0}+\log\frac{v_0\eta e^{\frac{|{\Delta_1}_{ij}|}{v_0}-\frac{|{\Delta_1}_{ij}|}{v_1}}+v_1{(1-\eta)}}{v_0\eta+v_1{(1-\eta)}}\Big)\right)\right),\\
\text{(C)}=&\frac{n}{8}k_1^2\|{\Delta_1}\|_F^2+\sum_{i\ne j, \in \mathcal{B}}\frac{1}{4}\text{pen}_{SS}{''}\left(\tilde{\theta}_{ij}+v{\Delta_1}_{ij}\right){\Delta_1}_{ij}^2.
\end{split}
\end{equation*}

Next, we show that all three terms, (A), (B), and (C), are non-negative. 
\begin{itemize}
\item $\text{(A)}=0$ because of the way $\tilde{\Theta}$ is constructed.

\item $\text{(C)} \ge 0$ by the property of $\text{pen}_{SS}{''}(\delta)$ stated before.

\item For term (B), we will first bound $s_{ij}-{\tilde{\Theta}^{-1}}_{ij}$: 
\begin{equation*}
\begin{split}
|s_{ij}-{\tilde{\Theta}^{-1}}_{ij}|&\le |s_{ij}-{\sigma}^0_{ij} |+|{\tilde{\Theta}^{-1}}_{ij}-\sigma^0_{ij}|\\
&\le C_1\sqrt{\frac{\log p}{n}}+M_{\Sigma^0}^22\left(C_1+C_3\right)M_{\Gamma^0}\sqrt{\frac{\log p}{n}}+\frac{3}{2}dM^3_{\Sigma^{0}}\left(2(C_1+C_3)M_{\Gamma^0}\sqrt{\frac{\log p}{n}}\right)^2\\
&\le \left(C_1+M_{\Sigma^0}^22\left(C_1+C_3\right)M_{\Gamma^0}+
6\left(C_1+C_3\right)^2dM_{\Gamma^0}^2M_{\Sigma^0}^3/M\right)\sqrt{\frac{\log p}{n}},
\end{split}
\end{equation*}
where the second line is due to  Lemma \ref{lemma:5}.

Next, we bound the fraction after the $\log$ function in $\text{(B)}$. For simplicity, denote it by $f({\Delta_1}_{ij})$. Since $1/v_0-1/v_1>0$, $f({\Delta_1}_{ij})$ is a monotone function of  ${\Delta_1}_{ij}$ and $f({\Delta_1}_{ij})$ goes to $1$ as ${\Delta_1}_{ij}$ goes to $0$. That is, $f({\Delta_1}_{ij})$ can be arbitrary close to 0, when ${\Delta_1}_{ij}$ is sufficiently small. Therefore the second term after summation can be arbitrary close to ${\Delta_1}_{ij}/(nv_0)$.

So if  choosing $1/(nv_0)>C_1+M_{\Sigma^0}^22(C_1+C_3)M_{\Gamma^0}+
6(C_1+C_3)^2dM_{\Gamma^0}^2M_{\Sigma^0}^3/M$ and  $\epsilon>0$ sufficiently small,  we have (B)$>$0 when $\|\Delta_1\|_{\infty}\le \epsilon$. 
\end{itemize}

Combining the results above, we have shown that 
 there always exists a small $\epsilon>0$, such that  $G(\Delta_1)\ge 0$ for any $\|\Delta_1\|_{\infty}\le \epsilon$. That is, $\tilde{\Theta}$ is a local minimizer. So we have proved Theorem \ref{thm:proof}. 
\end{proof}

\begin{proof}[\textbf{Proof of Theorem \ref{Thm:estimate}}] 
\cite{cai2011constrained} have shown that the sample noise $\tilde{W}$ can be bounded by $\sqrt{\frac{\log p}{n}}$ times a constant with high probability for both exponential tail and polynomial tail (see the proofs of  their Theorem 1 and 4). That is, 
\begin{itemize}
\item When condition (C1) holds, 
$$\|\tilde{W}\|_{\infty} \le\eta_1^{-1}(2+\tau_0+\eta_1^{-1}K^2)\sqrt{\frac{\log p}{n}}$$
with probability greater than $1-2p^{-\tau_0}.$

\item When condition (C2) holds, 
$$\|\tilde{W}\|_{\infty} \le\sqrt{({\theta^0_{max}}+1)(4+\tau_0)\frac{\log p}{n}}, \quad {\theta^0_{max}}=\max_{ij}\theta_{ij},$$
with probability greater than $1-O(n^{-\delta_0/8}+p^{-\tau_0/2})$. 
\end{itemize}

With the results above on $\|\tilde{W}\|_{\infty}$ and Theorem \ref{thm:proof}, we have proven Theorem \ref{Thm:estimate}.
\end{proof}

\section*{Appendix B: Other Proofs}

\begin{proof}[\textbf{Proof of Lemma \ref{lemma:4}}] Show both $|\Delta_{\mathcal{B}}\|_{\infty}$ and $ \|\Delta_{\mathcal{B}^c}\|_{\infty}$ are bounded by $r$. Thus, $\|\Delta\|_{\infty} \le r.$

\begin{enumerate}
\item  By construction, 
$$\|\Delta_{\mathcal{B}^c}\|_{\infty}\le 2(C_1+C_3)M_{\Gamma^0}\sqrt{\log p/n}\le r.$$ 

\item The proof for $\|\Delta_{\mathcal{B}}\|_{\infty} \le r$ is inspired by \cite{ravikumar2011high}. Define $G(\Theta_{\mathcal{B}})=n\left(-\Theta_{\mathcal{B}}^{-1}+S_{\mathcal{B}}\right)/2+Z_{\mathcal{B}}$. By definition, the set of $\Theta_{\mathcal{B}}$ that satisfies $G(\Theta_{\mathcal{B}})=0$ is the set $\mathcal{A}$.
Consider a mapping $F$ from $\mathbb{R}^{|{\mathcal{B}}|}\rightarrow \mathbb{R}^{|{\mathcal{B}}|}$:
\begin{equation}
F\left(\text{vec}\left(\Delta_{\mathcal{B}}\right)\right)=\frac{2}{n}\Big(-{\Gamma^{0}}_{{\mathcal{B}}{\mathcal{B}}}^{-1}\text{vec}\left(G\left(\Theta_{\mathcal{B}}^0+\Delta_{\mathcal{B}}\right)\right)\Big)+\text{vec}\left(\Delta_{\mathcal{B}}\right).
\end{equation}
By construction, $F\left(\text{vec}\left(\Delta_{\mathcal{B}}\right)\right)=\text{vec}\left(\Delta_{\mathcal{B}}\right)$ if and only if $G\left(\Theta_{\mathcal{B}}^0+\Delta_{\mathcal{B}}\right)=G\left(\Theta_{\mathcal{B}}\right)=0.$

Let $\mathbb{B}\left(r\right)$ denote the $\ell_{\infty}$ ball in $\mathbb{R}^{|{\mathcal{B}}|}$. 
If we could show that $F\left(\mathbb{B}\left(r\right)\right)\subseteq \mathbb{B}\left(r\right)$, then because $F$ is continuous and $\mathbb{B}\left(r\right)$ is convex and compact,  by Brouwer's fixed point theorem, there exists a fixed point $\text{vec}(\Delta_{\mathcal{B}})\in \mathbb{B}(r)$. Thus $\|\Delta_{\mathcal{B}}\|_{\infty}\le r$.

Let $\Delta\in\mathbb{R}^{p\times p}$ denote the zero-padded matrix, equal to $\Delta_{\mathcal{B}}$ on ${\mathcal{B}}$ and zero on ${\mathcal{B}}^c$.
\begin{equation*}
\begin{split}
F\left(\text{vec}\left(\Delta_{\mathcal{B}}\right)\right)&=\frac{2}{n}\Big(-{\Gamma^{0}}_{{\mathcal{B}}{\mathcal{B}}}^{-1}\text{vec}\left(G\left(\Theta_{\mathcal{B}}^0+\Delta_{\mathcal{B}}\right)\right)\Big)+\text{vec}\left(\Delta_{\mathcal{B}}\right)\\
&=-{\Gamma^{0}}_{{\mathcal{B}}{\mathcal{B}}}^{-1}\Big(\left(-(\Theta^0+\Delta)_{\mathcal{B}}^{-1}+S_{\mathcal{B}}\right)+\frac{2}{n}Z_{\mathcal{B}}\Big)+\text{vec}\left(\Delta_{\mathcal{B}}\right)\\
&=-{\Gamma^{0}}_{{\mathcal{B}}{\mathcal{B}}}^{-1}\Big(-\left(\Theta^0+\Delta\right)_{\mathcal{B}}^{-1}+{\Theta^{0}_{\mathcal{B}}}^{-1}-{\Theta^{0}_{\mathcal{B}}}^{-1}+S_{\mathcal{B}})+\frac{2}{n}Z_{\mathcal{B}}\Big)+\text{vec}\left(\Delta_{\mathcal{B}}\right)\\
&={\Gamma^{0}}_{{\mathcal{B}}{\mathcal{B}}}^{-1}\text{vec}\left(\Theta^{0^{-1}}\Delta\Theta^{0^{-1}}\Delta J\Theta^{0^{-1}}\right)_{\mathcal{B}}-{\Gamma^{0}}_{{\mathcal{B}}{\mathcal{B}}}^{-1}\left(\text{vec}\left(W_{\mathcal{B}}+\frac{2}{n}Z_{\mathcal{B}}\right)\right).
\end{split}
\end{equation*}
Denote 
\begin{eqnarray*}
\mathbf{I} &=& {\Gamma^{0}}_{{\mathcal{B}}{\mathcal{B}}}^{-1}\text{vec}\left(\Theta^{0^{-1}}\Delta\Theta^{0^{-1}}\Delta J\Theta^{0^{-1}}\right)_{\mathcal{B}} \\
\mathbf{II}&=&{\Gamma^{0}}_{{\mathcal{B}}{\mathcal{B}}}^{-1}\left(\text{vec}\left(W_{\mathcal{B}}+\frac{2}{n}Z_{\mathcal{B}}\right)\right).
\end{eqnarray*}
Then $F\left(\text{vec}\left(\Delta_{\mathcal{B}}\right)\right)\le \|\mathbf{I}\|_{\infty}+\|\mathbf{II}\|_{\infty}$. So it suffices to  show $\|\mathbf{I}\|_{\infty}+\|\mathbf{II}\|_{\infty}\le r$.

For the first relationship, we have
\begin{equation*}
\begin{split}
\|\mathbf{I}\|_{\infty}&\le \vertiii{{\Gamma^{0}}^{-1}_{{\mathcal{B}}{\mathcal{B}}}}_{\infty}\|\text{vec}(\Theta^{0^{-1}}\Delta\Theta^{0^{-1}}\Delta J\Theta^{0^{-1}})_{\mathcal{B}}\|_{\infty}\\
&\le M_{\Gamma^0}\|R(\Delta)\|_{\infty}\\
&\le \frac{3}{2}dM_{\Gamma^0}M_{\Sigma^0}^3\|\Delta\|_{\infty}^2,
\end{split}
\end{equation*}
where the last inequality is due to  $\|\Delta\|_{\infty}\le r\le 1/(3M_{\Sigma^0}d)$ and Lemma 5 from \cite{ravikumar2011high}. Since $r\le1/(3dM_{\Gamma^0}M_{\Sigma^0}^3)$, we further have $\|\mathbf{I}\|_{\infty}\le r/2.$

By assumption, $\min{|\theta_{\mathcal{B} \cap \mathcal{D}^c}^{0}|}\ge r+\delta$, thus when $\|\Delta\|_{\infty}\le r$, $\min|{\theta}_{\mathcal{B}\cap {\mathcal{D}^c}}|\ge\delta$, since $\text{pen}^{'}_{SS}\left(|\theta|\right)$ is monotonic decreasing, we have $\|Z_{\mathcal{B}\cap\mathcal{D}^c}\|_{\infty}\le\frac{1}{2}\text{pen}^{'}_{SS}(\delta)$. Thus, for the second relationship, we have 
\begin{equation*}
\begin{split}
\|\mathbf{II}\|_{\infty}&\le{\Gamma^{0}}^{-1}_{{\mathcal{B}}{\mathcal{B}}} \Big(\|W\|_{\infty}+\frac{2}{n}\max \left( \frac{1}{2}\text{pen}^{'}_{SS}(\delta),\tau\right)\Big)\\
& \le M_{\Gamma^0}\left(\|W\|_{\infty}+\frac{2}{n}\max \left( \frac{1}{2}\text{pen}^{'}_{SS}(\delta),\tau\right)\right) \le r/2
\end{split}
\end{equation*}
by assumption.
\end{enumerate}

Thus, there exists a point $\tilde{\Theta}$ such that $\|\tilde{\Theta}-\Theta^0\|_\infty\le r.$

Because $\|\tilde{\Theta}\|_2\le \|\tilde{\Theta}-\Theta^0\|_2+\|\Theta^0\|_2$ and $\|\tilde{\Theta}-\Theta^0\|_2\le \vertiii{\tilde{\Theta}-\Theta^0}_\infty\le dr$, we have $\|\tilde{\Theta}\|_2\le 1/k_1+dr<B.$ Because $dr < \frac{1}{3M_{\Sigma^0}} < \frac{1}{3}\lambda_{\min}(\Theta^0)$, we have $\lambda_{\min}(\tilde{\Theta}) >0.$So it is inside $\mathcal{A}$ by assumption. That is, $\mathcal{A}$ is non empty. 
\end{proof}

\begin{proof} [\textbf{Proof of Lemma \ref{lemma:5}}] 
Since there are only $p+s$ nonzero entries, we prove (\ref{eq1:lemma5}):
$$\|\tilde{\Theta}-\Theta^0\|_{F}=\sqrt{\sum_{(i,j)\in S_g}(\tilde{\theta}_{ij}-\theta_{ij}^0)^2}\le r\sqrt{p+s}.$$
Since there are at most $d$ nonzero entries in each column of $\Theta$ and $\Theta$ is symmetric,
$$\|\tilde{\Theta}-\Theta^0\|_{2}\le \vertiii{\tilde{\Theta}-\Theta^0}_{\infty}\le rd.$$
In addition, since the $\ell_\infty/\ell_\infty$ operator norm is bounded by Frobenius norm, we prove (\ref{eq2:lemma5}). We skip the proof for (\ref{eq3:lemma5}), which is nearly identical to Corollary 4 in \cite{ravikumar2008high}. \end{proof}

\begin{proof}[\textbf{Proof of Theorem \ref{Thm:select}}](Selection consistency)\\
Recall 
\begin{equation}
\begin{split}
\log\frac{p_{ij}}{1-p_{ij}}&=\Big(\log\frac{v_0\eta}{v_1(1-\eta)}-\frac{|\tilde{\theta}_{ij}|}{v_1}+\frac{|\tilde{\theta}_{ij}|}{v_0}\Big)\\
&=\Big(-\log\frac{v_1(1-\eta)}{v_0\eta}-\frac{|\tilde{\theta}_{ij}|}{v_1}+\frac{|\tilde{\theta}_{ij}|}{v_0}\Big).
\end{split}
\end{equation}

\begin{itemize}
\item When $\theta^0_{ij}=0$, 
by constructor, $\tilde{\theta}_{ij}=0.$ Then with our choice of $v_1(1-\eta)/\left(v_0\eta\right)$, 
$$\log\frac{p_{ij}}{1-p_{ij}}\rightarrow -\infty.$$

\item When $\theta^0_{ij}\ne0$, we have
\begin{equation}
\begin{split}
\log\frac{p_{ij}}{1-p_{ij}}&=\Big(\log\frac{v_0\eta}{v_1(1-\eta)}-\frac{|\theta_{ij}|}{v_1}+\frac{|\theta_{ij}|}{v_0}\Big)\\
&\ge \left(-\log\frac{v_1(1-\eta)}{v_0\eta}+\left(\frac{1}{v_0}-\frac{1}{v_1}\right)\left(|\theta^0_{ij}|-|\theta^0_{ij}-\theta_{ij}|\right)\right)\\
&\ge -\log\frac{v_1(1-\eta)}{v_0\eta}+(C_4-C_3)\left(K_0-2(C_1+C_3)M_{\Gamma^0}\right)\log p.
\end{split}
\end{equation}
Then with our choice of $v_1(1-\eta)/ (v_0\eta)$, 
$$\log\frac{p_{ij}}{1-p_{ij}}\rightarrow +\infty.$$
\end{itemize}

\end{proof}

\begin{proof}[\textbf{Proof of Theorem \ref{thm:sym}}] 
The estimate of the precision matrix is symmetric due to construction. 

Next we show that the estimate is ensured to be positive definite.  Assume $\Theta^{(t)}$, the $t$-th update of the estimate is positive definite. Apparently, this assumption is satisfied with $t=0$ since  the initial estimate $\Theta^{(0)}$ is positive definite. 

Then it suffices to show that $\det(\Theta^{(t+1)})\succ0$. WLOG, assume we update the last column of $\Theta$ in the $(t+1)$-th iteration. Using Schur complements, we have
$$ \det\left(\Theta^{(t+1)}\right)=\det\left(\Theta_{11}^{(t)}\right)\left(\theta_{22}^{(t+1)}-\theta_{12}^{{(t+1)^T}}\Theta_{11}^{{(t)}^{-1}}\theta_{12}^{(t+1)}\right).$$
Because $\det(\Theta^{(t)})\succ0$, we have $\det\left(\Theta_{11}^{(t)}\right)>0.$ Further, 
 the updating rule of our algorithm ensures that 
 $$\left(\theta_{22}^{(t+1)}-\theta_{12}^{{(t+1)^T}}\Theta_{11}^{{(t)}^{-1}}\theta_{12}^{(t+1)}\right)=\frac{1}{w_{22}^{(t+1)}}>0.$$
Thus, $\det\left(\Theta^{(t+1)}\right) >0$. 
\end{proof}

\section*{Appendix C: Checking $\|\Theta\|_2\le B$.}
Algorithm 1 involves checking the spectral norm constraint $\|\Theta\|_2\le B$ after every column update of  $\Theta.$ Computing $\|\Theta\|_2$ can be computationally intensive, however, since we only change one column (and corresponding one row) at a time, the constraint can be checked without calculating $\|\Theta\|_2$ every time. Suppose we know $\|\Theta^{(t)}\|_2$ (or an upper bound) at the previous step, and denote $\Delta^{(t)}:= \Theta^{(t+1)} - \Theta^{(t)}$ to be the difference between the estimates after one column update. In order to check the bound, it is sufficient to make sure that $\|\Theta ^{(t)}\|_2+\|\Delta^{(t)}\|_2<B$. It is easy to check this constraint because $\|\Delta^{(t)}\|_2$ is a rank two matrix with its maximum eigenvalue available in closed form. Only when $\|\Theta ^{(t)}\|_2+\|\Delta\|_2$ exceeds $B$, we will need to recalculate $\|\Theta^{(t+1)}\|_2$ again. 

\bibliographystyle{apalike}
\bibliography{main}
\bigskip

\end{document}